\documentclass[11pt]{article}
\usepackage{graphicx,amsmath,amsfonts,amssymb,bm,hyperref,url,breakurl,epsfig,epsf,color,fullpage,MnSymbol,mathbbol, fmtcount,algorithmic,algorithm, semtrans
} 

\usepackage{titlesec}

\usepackage{tikz}
\usepackage{pgfplots}

\usetikzlibrary{pgfplots.groupplots}

\titleformat{\paragraph}
{\normalfont\normalsize\bfseries}{\theparagraph}{1em}{}
\titlespacing*{\paragraph}
{0pt}{3.25ex plus 1ex minus .2ex}{1.5ex plus .2ex}

\usepackage{caption,subcaption}

\usepackage[bottom,hang,flushmargin]{footmisc} 

\usepackage{hyperref}
\definecolor{darkred}{RGB}{150,0,0}
\definecolor{darkgreen}{RGB}{0,150,0}
\definecolor{darkblue}{RGB}{0,0,200}
\hypersetup{colorlinks=true, linkcolor=darkred, citecolor=darkgreen, urlcolor=black}

\newtheorem{theorem}{Theorem}[section]
\newtheorem{lemma}[theorem]{Lemma}
\newtheorem{corollary}[theorem]{Corollary}

\newtheorem{definition}[theorem]{Definition}

\newtheorem{condition}[theorem]{Condition}

\newcommand{\vb}{{\mtx{v}}}
\newcommand{\tn}[1]{\left\|#1\right\|_{\ell_2}}
\newcommand{\Cc}{\mathcal{C}}
\newcommand{\ub}{\mtx{u}}
\newcommand{\Bc}{\mathcal{B}}
\newcommand{\Pc}{\mathcal{P}}
\newcommand{\Dc}{\mathcal{D}}

\newcommand{\des}{{\z}}
\newcommand{\z}{{\mtx{z}}}

\newcommand{\h}{\vct{h}}

\newcommand{\fronorm}[1]{\left\|#1\right\|_{F}}
\newcommand{\onenorm}[1]{\left\|#1\right\|_{\ell_1}}
\newcommand{\twonorm}[1]{\left\|#1\right\|_{\ell_2}}

\newcommand{\infnorm}[1]{\left\|#1\right\|_{\ell_\infty}}

\newcommand{\abs}[1]{\left|#1\right|}

\newcommand{\x}{\vct{x}}

\newcommand{\R}{\mathbb{R}}
\newcommand{\C}{\mathbb{C}}

\newcommand{\sgn}[1]{\textrm{sgn}(#1)}

\newcommand{\E}{\operatorname{\mathbb{E}}}

\newcommand{\vct}[1]{\bm{#1}}
\newcommand{\mtx}[1]{\bm{#1}}

\definecolor{ejc}{RGB}{0,0,255}

\numberwithin{equation}{section} 

\def \endprf{\hfill {\vrule height6pt width6pt depth0pt}\medskip}
\newenvironment{proof}{\noindent {\bf Proof} }{\endprf\par}

\title{Structured signal recovery from quadratic measurements:\\
Breaking sample complexity barriers via nonconvex optimization}
\author{Mahdi Soltanolkotabi\\Ming Hesieh Department of Electrical Engineering\\University of Southern California, Los Angeles, CA, 90089}
\date{February 20, 2017}
\begin{document}
\maketitle
\begin{abstract}
This paper concerns the problem of recovering an unknown but structured signal $\vct{x}\in\R^n$ from $m$ quadratic measurements of the form $\vct{y}_r=\abs{\langle\vct{a}_r,\vct{x}\rangle}^2$ for $r=1,2,\ldots,m$. We focus on the under-determined setting where the number of measurements is significantly smaller than the dimension of the signal ($m<<n$). We formulate the recovery problem as a nonconvex optimization problem where prior structural information about the signal is enforced through constrains on the optimization variables. We prove that projected gradient descent, when initialized in a neighborhood of the desired signal, converges to the unknown signal at a linear rate. These results hold for \emph{any} constraint set (convex or nonconvex) providing convergence guarantees to the global optimum even when the objective function and constraint set is nonconvex. Furthermore, these results hold with a number of measurements that is only a constant factor away from the minimal number of measurements required to uniquely identify the unknown signal. Our results provide the first provably tractable algorithm for this data-poor regime, breaking local sample complexity barriers that have emerged in recent literature. In a companion paper we demonstrate favorable properties for the optimization problem that may enable similar results to continue to hold more globally (over the entire ambient space). Collectively these two papers utilize and develop powerful tools for uniform convergence of empirical processes that may have broader implications for rigorous understanding of constrained nonconvex optimization heuristics. The mathematical results in this paper also pave the way for a new generation of data-driven phase-less imaging systems that can utilize prior information to significantly reduce acquisition time and enhance image reconstruction, enabling nano-scale imaging at unprecedented speeds and resolutions. 
\end{abstract}
\section{Introduction}
Signal reconstruction from quadratic measurements is at the heart of many applications in signal and image processing. In this problem we acquire quadratic measurements of the form
\begin{align}
\label{meas}
y_r=\abs{\langle\vct{a}_r,\vct{x}\rangle}^2,\quad r=1,2,\ldots,m,
\end{align}
from an unknown, structured signal $\vct{x}\in\C^n$. Here, $\vct{a}_r\in\C^n$ are known sampling vectors and $y_r\in\R$ are observed measurements. Such quadratic signal recovery problems are of interest in a variety of domains ranging from combinatorial optimization, to wireless communications and imaging. Focusing on signal processing applications, recovering a signal from measurements of the form \eqref{meas} is usually referred to as the \emph{generalized phase retrieval problem}. The connection with phase retrieval is due to the fact that optical detectors, specially at small wavelengths, can often only record the intensity of the light field and not its phase.  Indeed, the acquired measurements in many popular coherent diffraction imaging systems such as those based on Ptychography or phase from defocus are of the form \eqref{meas}, with $\vct{x}$ corresponding to the object of interest, $\vct{a}_r$ modulated sinusoids, and $y_r$ the recorded data.

Given the ubiquity of the generalized phase retrieval problem in signal processing, over the years many heuristics have been developed for its solution. On the one hand, invention of new X-ray sources and new experimental setups that enable recording and reconstruction of non-crystalline objects has caused a major revival in the use of phase retrieval techniques in imaging. On the other hand, the last five years has also witnessed tremendous progress in terms of providing rigorous mathematical guarantees for the performance of some classical heuristics such as alternating minimization \cite{netrapalli2013phase, waldspurger2016phase} as well as newer ones based on semidefinite programing \cite{candes_strohmer} and Wirtinger flows \cite{WF} and its variants \cite{chen2015solving, zhang2016provable, zhang2016reshaped, cai2015optimal, wang2016solving}. We shall review all these algorithms and mathematical results in greater detail in Section \ref{PART}. These results essentially demonstrate that a signal of dimension $n$ can be recovered efficiently and reliably from the order of $m\gtrsim n$ generic quadratic measurements of the form \eqref{meas}.

The recent surge of applied and theoretical activity regarding phaseless imaging is in part driven by the hope that it will eventually lead to successful imaging of large protein complexes and biological specimens enabling live imaging of bio-chemical activities at the molecular level. Furthermore, phaseless imaging techniques increasingly play a crucial role in emerging national security applications aimed at monitoring electronic products that are intended for military or infrastructure use so as to ensure these products do not contain secret backdoors granting foreign governments cyber access to vital US infrastructure. Despite the incredible progress discussed earlier on both applied and mathematical fronts, major challenges impede the use of such techniques to these emerging domains. One major challenge is that acquiring measurements of large specimens at high resolutions (corresponding to very short wavelengths) require time consuming and expensive measurements. To be concrete, the most modern phase-less imaging setups require image acquisition times exceeding $2500$ days for imaging a $1$ micrometer $\times 1$ micrometer specimen at $10$nm resolution! 

To overcome these challenges, in this paper we aim to utilize a-priori structural information available about the signal to reduce the required number of quadratic measurements of the form \eqref{meas}. Indeed, in the application domains discussed above there is a lot of a-priori knowledge available that can be utilized to reduce acquisition time and enhance image reconstruction. For example, images of electronic chips are extremely structured e.g. piecewise constant and often projections of 3D rectilinear models. While, historically various a-priori information such as non-negativty has been used to enhance image reconstruction, such simple forms of structural information are often not sufficient. Complicating the matter further our mathematical understanding of how well even simple forms of a-priori information can enhance reconstruction is far from complete. To be concrete, assume we know the signal of interest is sparse e.g.~it has at most $s$ nonzero entries. In this case for known tractable algorithms to yield accurate solutions the number of generic measurements must exceed $c\frac{s^2}{\log n}$ with $c$ a constant \cite{oymak2015simultaneously}. This is surprising, as the degrees of freedom of an $s$-sparse vector is of the order $s$ and based on the compressive sensing literature one expects to be able to recover the signal $\vct{x}$ from the order of $s\log(n/s)$ generic quadratic measurements. In fact, it is known that on the order of $s\log(n/s)$ generic quadratic measurements uniquely specify the unknown signal up to a global phase factor. However, it is not known whether a tractable algorithm can recover the signal from such minimal number of generic quadratic measurements. 

The above example demonstrates a significant gap in our ability to utilize prior structural assumptions in phase retrieval problems so as to reduce the required number of measurements or \emph{sample complexity}. This is not an isolated example, and such gaps hold more generally for a variety of problems and structures (see \cite{oymak2015simultaneously, amini2008high, deshpande2015improved, deshpande2014sparse, deshpande2015finding} for more details on related gaps). The emergence of such sample complexity ``barriers" is quite surprising as in many cases there is no tractable algorithm known to close this gap. In fact, for some problems such as sparse PCA it is known that closing this gap via a computationally tractable approach will yield tractable algorithms for notoriously difficult problems such as planted clique \cite{berthet2013complexity}. 

\section{Minimizing (non)convex objectives with (non)convex constraints}
We wish to discern an unknown but ``structured" signal $\vct{x}\in\C^n$ from $m$ quadratic measurements of the form $\vct{y}_r=\abs{\langle\vct{a}_r,\vct{x}\rangle}^2$, for $r=1,2,\ldots,m$. However, in the applications of interest typically the number of equations $m$ is significantly smaller than the number of variables $n$
so that there are infinitely many solutions obeying the quadratic constraints. However, it may still be
possible to recover the signal by exploiting knowledge of its structure. To this aim, let $\mathcal{R}:\R^n\rightarrow\R$ be a cost function that reflects some notion of ``complexity" of the ``structured" solution. It is
then natural to use the following optimization problem to recover the signal.
\begin{align}
\label{opt}
\underset{\vct{z}\in\C^n}{\text{minimize}}\quad \mathcal{L}(\vct{z}):=\frac{1}{m}\sum_{r=1}^m \ell\left(\sqrt{\vct{y}_r},\abs{\vct{a}_r^*\vct{z}}\right)\quad\text{subject to}\quad \mathcal{R}(\vct{z})\le \mathcal{R}(\vct{x}).
\end{align}
Here, $\ell(\sqrt{\vct{y}_r},\abs{\vct{a}_r^*\vct{z}})$ is a loss function measuring the misfit between the measurements $\sqrt{\vct{y}_r}$ and the data model and $\mathcal{R}$ is a regularization function that reflects known prior knowledge about the signal. A natural approach to solve this problem is via projected gradient type updates of the form
\begin{align}
\label{iters}
\vct{z}_{\tau+1}=\mathcal{P}_{\mathcal{K}}\left(\vct{z}_\tau-\mu_\tau\nabla \mathcal{L}(\vct{z}_\tau)\right).
\end{align}
Here, $\nabla \mathcal{L}$ is the Wirtinger derivative of $\mathcal{L}$ (see \cite[Section 6]{WF} for details) and $\mathcal{P}_{\mathcal{K}}(\vct{z})$ denotes the projection of $\vct{z}\in\C^n$ onto the set 
\begin{align}
\label{setk}
\mathcal{K}=\{\vct{w}\in\C^n:\mathcal{R}(\vct{w})\le \mathcal{R}(\vct{x})\}.
\end{align} 
Following \cite{WF}, we shall refer to this iterative procedure as the Projected Wirtinger Flow (PWF) algorithm. 

A-priori it is completely unclear why the iterative updates \eqref{iters} should converge as not only the loss function may be nonconvex but also the regularization function!  Efficient signal reconstruction from nonlinear measurements in this high-dimensional setting poses new challenges:
\begin{itemize}
\item When are the iterates able to escape local optima and saddle points and converge to global optima?
\item How many measurements do we need? Can we break through the barriers faced by convex relaxations?
\item How does the number of measurements depend on the a-priori prior knowledge available about the signal? What regularizer is best suited to utilizing a particular form of prior knowledge? 
\item How many passes (or iterations) of the algorithm is required to get to an accurate solution?
\end{itemize}
At the heart of answering these questions is the ability to predict convergence behavior/rate of (non)convex constrained optimization algorithms. 

\section{Precise measures for statistical resources}
Throughout the rest of the paper we assume that the signal $\vct{x}\in\R^n$ and the measurement vectors $\vct{a}_r\in\R^n$ are all real-valued. For sake of brevity we have focused our attention to this real-valued case. However, we note that all of our definitions/results trivially extend to the complex case. We wish to characterize the rates of convergence for the projected gradient updates \eqref{iters} as a function of the number of samples, the available prior knowledge and the choice of the regularizer. To make these connections precise and quantitative we need a few definitions. Naturally the required number of samples for reliable signal reconstruction depends on how well the regularization function $\mathcal{R}$ can capture the properties of the unknown signal $\vct{x}$. For example, if we know our unknown parameter is approximately sparse naturally using an $\ell_1$ norm for the regularizer is superior to using an $\ell_2$ regularizer. To quantify this capability we first need a couple of standard definitions which we adapt from \cite{oymak2015sharp, oymak2016fast}.
\begin{definition}[Descent set and cone] \label{decsetcone} The \emph{set of descent} of  a function $\mathcal{R}$ at a point $\vct{x}$ is defined as
\begin{align*}
{\cal D}_{\mathcal{R}}(\vct{x})=\Big\{\vct{h}:\text{ }\mathcal{R}(\vct{x}+\vct{h})\le \mathcal{R}(\vct{x})\Big\}.
\end{align*}
The \emph{cone of descent} is defined as a closed cone $\mathcal{C}_{\mathcal{R}}(\vct{x})$ that contains the descent set, i.e.~$\mathcal{D}_{\mathcal{R}}(\vct{x})\subset\mathcal{C}_{\mathcal{R}}(\vct{x})$. The \emph{tangent cone} is the conic hull of the descent set. That is, the smallest closed cone $\mathcal{C}_{\mathcal{R}}(\vct{x})$ obeying $\mathcal{D}_{\mathcal{R}}(\vct{x})\subset\mathcal{C}_{\mathcal{R}}(\vct{x})$.
\end{definition}
We note that the capability of the regularizer $\mathcal{R}$ in capturing the properties of the unknown signal $\vct{x}$ depends on the size of the descent cone $\mathcal{C}_{\mathcal{R}}(\vct{x})$. The smaller this cone is the more suited the function $\mathcal{R}$ is at capturing the properties of $\vct{x}$. To quantify the size of this set we shall use the notion of mean width.
\begin{definition}[Gaussian width]\label{Gausswidth} The Gaussian width of a set $\mathcal{C}\in\R^p$ is defined as:
\begin{align*}
\omega(\mathcal{C}):=\mathbb{E}_{\vct{g}}[\underset{\vct{z}\in\mathcal{C}}{\sup}~\langle \vct{g},\vct{z}\rangle],
\end{align*}
where the expectation is taken over $\vct{g}\sim\mathcal{N}(\vct{0},\mtx{I}_p)$. Throughout we use $\mathcal{B}^n/\mathbb{S}^{n-1}$ to denote the the unit ball/sphere of $\R^n$.
\end{definition}
We now have all the definitions in place to quantify the capability of the function $\mathcal{R}$ in capturing the properties of the unknown parameter $\vct{x}$. This naturally leads us to the definition of the minimum required number of samples.
\begin{definition}[minimal number of samples]\label{PTcurve}
Let $\mathcal{C}_{\mathcal{R}}(\vct{x})$ be a cone of descent of $\mathcal{R}$ at $\vct{x}$. We define the minimal sample function as
\begin{align*}
\mathcal{M}(\mathcal{R},\vct{x})=\omega^2(\mathcal{C}_{\mathcal{R}}(\vct{x})\cap\mathcal{B}^n).
\end{align*}
We shall often use the short hand $m_0=\mathcal{M}(\mathcal{R},\vct{x})$ with the dependence on $\mathcal{R},\vct{x}$ implied.  
\end{definition}
We note that $m_0$ is exactly the minimum number of samples required for structured signal recovery from linear measurements when using convex regularizers \cite{Cha, McCoy}. Specifically, the optimization problem
\begin{align}
\label{lininv}
\sum_{r=1}^m \left(y_r-\langle \vct{a}_r,\vct{x}\rangle\right)^2\quad\text{subject to}\quad \mathcal{R}(\vct{z})\le\mathcal{R}(\vct{x}),
\end{align}
succeeds at recovering the unknown signal $\vct{x}$ with high probability from $m$ measurements of the form $\vct{y}_r=\langle \vct{a}_r,\vct{x}\rangle$ if and only if $m\ge m_0$.\footnote{We would like to note that $m_0$ only approximately characterizes the minimum number of samples required. A more precise characterization is $\phi^{-1}(\omega^2(\mathcal{C}_{\mathcal{R}}(\vct{x})\cap\mathcal{B}^n))\approx \omega^2(\mathcal{C}_{\mathcal{R}}(\vct{x})\cap\mathcal{B}^n)$ where $\phi(t)=\sqrt{2}\frac{\Gamma\left(\frac{t+1}{2}\right)}{\Gamma\left(\frac{t}{2}\right)}\approx\sqrt{t}$. However, since our results have unspecified constants we avoid this more accurate characterization.} While this result is only known to be true for convex regularization functions we believe that $m_0$ also characterizes the minimal number of samples even for nonconvex regularizers in \eqref{lininv}. See \cite{oymak2015sharp} for some results in the nonconvex case as well as the role this quantity plays in the computational complexity of projected gradient schemes for linear inverse problems. Given that in phase-less imaging we have less information (we loose the phase of the linear measurements) we can not hope to recover structured signals from $m\le m_0$ when using \eqref{opt}. Therefore, we can use $m_0$ as a lower-bound on the minimum number of measurements required for projected gradient descent iterations \eqref{iters} to succeed in recovering the signal of interest.


\section{Nonconvex regularization examples}
Next we provide two examples of (non)convex regularizers which are of interest in phase-less imaging applications. These two simple examples are meant to highlight the importance of nonconvex regularizers in imaging. However, our theoretical framework is by no means limited to these simple examples and can deal with significantly more complicated nonconvex regularizers capturing much more nuanced forms of prior structure.

\begin{itemize}
\item \textbf{piecewise constant structure.} One a-priori structural information available in many phase-less imaging applications is that images tend to be piecewise constant. For example, contiguous parts of biological specimens are made up of the same tissue and exhibit the same behavior under electro-magnetic radiations. Similarly, images of electronic chips are often projections of piecewise constant, 3D rectilinear models. Let $\vct{z}\in\R^{n_1\times n_2}$ denote a $2D$ image consisting of an array of pixels. A popular approache for exploiting piecewise-constant structures is to use total variation regularization functions. Two common choices are the isotropic and anisotropic totoal variation regularizations defined as
\begin{align*}
\mathcal{R}_{iso}(\vct{z})=&\sum_{i,j}\left(\sqrt{\abs{\vct{z}_{i+1,j}-\vct{z}_{i,j}}^2+\abs{\vct{z}_{i,j+1}-\vct{z}_{i,j}}^2}\right)^p,\\
\mathcal{R}_{ani}(\vct{z})=&\sum_{i,j} \abs{\vct{z}_{i+1,j}-\vct{z}_{i,j}}^p+\abs{\vct{z}_{i,j+1}-\vct{z}_{i,j}}^p.
\end{align*}
When $p=1$, these regularization functions are convex. However, for many image reconstruction tasks total variation regularization with $p<1$, despite being nonconvex, is significantly more effective at capturing piece-wise constant structure.
   
\item \textbf{discrete values.} Another form of a-priori knowledge that is sometimes available in imaging applications is possible discrete values the image pixels can take. For example, in imaging electronic chips the metallic composition of the different parts pre-determines the possible discrete values and are known in advance. Let $\vct{z}\in\R^{n_1\times n_2}$ denote a $2D$ image consisting of an array of pixels and assume the possible discrete values are $\{a_1, a_2, \ldots, a_k\}$. A natural regularization in this case is
\begin{align*}
\mathcal{R}(\vct{z})=\sum_{i,j} \mathcal{I}\left( \prod_{r=1}^k(\vct{z}_{ij}-a_r)\right),\quad\text{with}\quad \mathcal{I}(z)=
\left\{
	\begin{array}{ll}
		0  & \mbox{if } z=0 \\
		+\infty & \mbox{if } z\neq 0
	\end{array}.
\right.
\end{align*}
This regularization is convenient as projection onto its sub-level sets is easy and amounts to replacing each entry of the input vector/matrix with the closest discrete value from $\{a_1, a_2, \ldots, a_k\}$ (a.k.a.~hard thresholding). Of course this is not the only regularization function that can enforce discrete structures and in practice ``soft thresholding" variants may be more effective. Our framework can be used to analyze many such variants. Indeed, an interesting aspect of our results is that it allows us to understand what regularizer is best suited at enforcing a particular form of prior structure.
\end{itemize}

\section{Theoretical results for Projected Wirtinger Flows}
In this section we shall explain our main theoretical results. To this aim we need to define the distance to the solution set.
\begin{definition} Let $\vct{x}\in\R^n$ be any solution to the quadratic system $\vct{y}=\abs{\mtx{A}\vct{x}}^2$ (the signal we wish to recover). For each $\vct{z}\in\R^n$, define
\begin{align*}
\emph{dist}(\vct{z},\vct{x})=\min\left(\twonorm{\vct{z}-\vct{x}},\twonorm{\vct{z}+\vct{x}}\right).
\end{align*}
\end{definition}
As we mentioned earlier we are interested in recovering structured signal recovery problems from quadratic measurements via the optimization problem \eqref{opt}. Naturally the convergence/lack of convergence as well as the rate of convergence of projected Wirtinger Flow iterates \eqref{iters} depends on the loss function $\ell$. We now discuss our theoretical results for two different loss functions.

\subsection{Intensity-based Wirtinger Flows with convex regularizers}\label{inten}
Our first result focuses on a quadratic loss function applied to intensity measurements, i.e.~$\ell(x,y)=\frac{1}{2}\left(x^2-y^2\right)^2$ in \eqref{opt}. In this case, the optimization problem takes the form
\begin{align}
\label{optsec}
\underset{\vct{z}\in\R^n}{\text{minimize}}\quad \mathcal{L}_I(\vct{z}):=\frac{1}{4m}\sum_{r=1}^m \left(\vct{y}_r-\abs{\vct{a}_r^*\vct{z}}^2\right)^2\quad\text{subject to}\quad \mathcal{R}(\vct{z})\le \mathcal{R}(\vct{x}).
\end{align}
Our first result studies the effectiveness of projected Wirtinger Flows on this objective. 
\begin{theorem}\label{Ithm}
Let $\vct{x}\in\R^n$ be an arbitrary vector and $\mathcal{R}:\R^n\rightarrow\R$ be a proper convex function. Suppose $\mtx{A}\in\R^{m\times n}$ is a Gaussian map and let $\vct{y}=\abs{\mtx{A}\vct{x}}^2\in\R^m$ be $m$ quadratic measurements. To estimate $\vct{x}$, start from a point $\vct{z}_0$ obeying
\begin{align}
\label{initIPWF}
\emph{dist}(\vct{z}_0,\vct{x})\le \frac{1}{8}\twonorm{\vct{x}},
\end{align}
and apply the Projected Wirtinger Flow (PWF) updates
\begin{align}
\label{myrealupdate}
\vct{z}_{\tau+1}=\mathcal{P}_{\mathcal{K}}\left(\vct{z}_\tau-\mu_\tau\nabla \mathcal{L}(\vct{z}_\tau)\right),
\end{align}
with $\mathcal{K}:=\{\vct{z}\in\R^n:\text{ }\mathcal{R}(\vct{z})\le \mathcal{R}(\vct{x})\}$. Also set the learning parameter sequence as $\mu_0=0$ and $\mu_\tau=\frac{\mu}{\twonorm{\vct{x}}}$\footnote{We note that $\twonorm{\vct{x}}$ can be trivially estimated from the measurements as $\frac{1}{m}\sum_{r=1}^m y_r^2\approx \twonorm{\vct{x}}^2$ and our proofs are robust to this misspecification. We avoid stating this variant for ease of reading.} for all $\tau=1,2,\ldots$ and assume $\mu\le c_1/n$ for some fixed numerical constant $c_1$. Furthermore, let $m_0=\mathcal{M}(\mathcal{R},\vct{x})$, defined by \ref{PTcurve}, be our lower bound on the number of measurements. Also assume
\begin{align}
\label{nummeaslin}
m>c m_0 \log n,
\end{align}
holds for a fixed numerical constant $c$. Then there is an event of probability at least $1-2/m-1/n-12e^{-\gamma m}$ such that on this event starting from any initial point obeying \eqref{initIPWF} the update \eqref{myrealupdate} satisfy
\begin{align}
\label{ratemin}
\emph{dist}(\vct{z}_\tau,\vct{x})\le\left(1-\frac{\mu}{125}\right)^{\frac{\tau}{2}}\emph{dist}(\vct{z}_0,\vct{x}).
\end{align}
\end{theorem}
As mentioned earlier $m_0$ is the minimal number of measurements required to recover a structured signal from linear measurements. $m_0$ also serves as a lower bound on structured signal recovery from quadratic measurements as they are even less informative (we loose sign information). Theorem \ref{Ithm} shows that PWF applied to quadratic loss using intensity measurements can (locally) reconstruct the signal with this minimal sample complexity (up to a constant and log factor). To be concrete consider the case where the unknown signal $\vct{x}$ is known to be $s$ sparse and we use $\mathcal{R}(\vct{z})=\onenorm{\vct{z}}$ as the regularizer in \eqref{optsec}. In this case it is known that $m_0\approx 2s\log(n/s)$ and Theorem \ref{Ithm} predicts that an $s$-sparse signal can (locally) be recovered from the order of $s\log(n/s)\log(n)$ measurements. This breaks through well-known barriers that have emerged for this problem in recent literature. Indeed, for known tractable convex relaxation schemes to yield accurate solutions the number of generic measurements must exceed $c\frac{s^2}{\log n}$ with $c$ a constant \cite{LiVoroninski2012, oymak2015simultaneously}. We also note that even recent nonconvex approaches such as \cite{cai2015optimal, wang2016sparse} have also not succeeded at breaking through this $s^2$ barrier even when an initialization obeying \eqref{initIPWF} is available.

The convergence guarantees provided above hold as long as PWF is initialized per \eqref{initIPWF} in a neighborhood of the unknown signal with relative error less than a constant. In this paper we are concerned only with local convergence properties of PWF and therefore do not provide an explicit construction for such an initialization. However, in a companion paper we demonstrate that the optimization problem \eqref{optsec} has certain favorable characteristics that may allow global convergence guarantees from any initialization using second order methods.

Another interesting aspect of the above result is that the rate of convergence is geometric. Specifically, to achieve a relative error of $\epsilon$ ($\twonorm{\vct{z}-\vct{x}}/\twonorm{\vct{x}}\le \epsilon$), the required number of iterations is $n\log(1/\epsilon)$. Note that the cost of each iteration depends on applying the matrix $\mtx{A}$ and its transpose $\mtx{A}^T$ which has computational complexity on the order of $\mathcal{O}(mn)$. This is assuming that the projection has negligible cost compared to the cost of applying $\mtx{A}/\mtx{A}^T$. This is the case for example for sparse signals when using the regularizer $\mathcal{R}(\vct{z})=\onenorm{\vct{z}}$. Therefore, in these cases to achieve a relative error of $\epsilon$ the total computational complexity of PWF is on the order of $\mathcal{O}\left(mn^2\log(1/\epsilon)\right)$.
 
Let us now discuss some ways in which this theorem is sub-optimal. Even though this theorem breaks through known sample complexity barriers, a natural question is whether it is possible to remove the log factor so as to have a sample complexity that is only a constant factor away from the minimum sample complexity of structured signal recovery from linear measurements. Another way in which the algorithm is sub-optimal is computational complexity. While the rate of convergence of PWF stated above is geometric, it is not linear. With a linear rate of convergence to achieve a relative error of $\epsilon$ the total computational complexity would be on the order of $\mathcal{O}\left(mn\log(1/\epsilon)\right)$ which is a factor of $n$ smaller than the guarantees provided by PWF. In the next section we will show how to close these gaps in sample complexity and computational complexity by using a different loss function in \eqref{optsec}. Finally, a major draw back of Theorem \ref{Ithm} it that it only applies to convex regularizers. In the next section we will show how to also remove this assumption so as to allow arbitrary nonconvex regularizers.

\subsection{Amplitude-based Wirtinger Flows with (non)convex regularizers}
Our second result focuses on a quadratic loss function applied to amplitude measurements, i.e.~$\ell(x,y)=\frac{1}{2}\left(x-y\right)^2$ in \eqref{opt}. In this case, the optimization problem takes the form
\begin{align}
\label{optsecA}
\underset{\vct{z}\in\R^n}{\text{minimize}}\quad \mathcal{L}_A(\vct{z}):=\frac{1}{2m}\sum_{r=1}^m \left(\sqrt{\vct{y}_r}-\abs{\vct{a}_r^*\vct{z}}\right)^2\quad\text{subject to}\quad \mathcal{R}(\vct{z})\le \mathcal{R}(\vct{x}).
\end{align}
One challenging aspect of the above loss function is that it is not differentiable and it is not clear how to run projected gradient descent. However, this does not pose a fundamental challenge as the loss function is differentiable except for isolated points and we can use the notion of generalized gradients to define the gradient at a non-differentiable point as one of the limits points of the gradient in a local neighborhood of the non-differentiable point. For the loss in \eqref{optsecA} the generalized gradient takes the form
\begin{align}
\label{gengrad}
\nabla \mathcal{L}_A(\vct{z}):=\frac{1}{m}\sum_{r=1}^m\left(\abs{\vct{a}_r^*\vct{z}}-\sqrt{\vct{y}_r}\right)\sgn{\vct{a}_r^*\vct{z}}\vct{a}_r.
\end{align}
\begin{theorem}\label{Athm}
Let $\vct{x}\in\R^n$ be an arbitrary vector and $\mathcal{R}:\R^n\rightarrow\R$ be a proper function (convex or nonconvex). Suppose $\mtx{A}\in\R^{m\times n}$ is a Gaussian map and let $\vct{y}=\abs{\mtx{A}\vct{x}}^2\in\R^m$ be $m$ quadratic measurements. To estimate $\vct{x}$, start from a point $\vct{z}_0$ obeying
\begin{align}
\label{initAPWF2}
\emph{dist}(\vct{z}_0,\vct{x})\le \frac{1}{15}\twonorm{\vct{x}},
\end{align}
and apply the Projected Wirtinger Flow (PWF) updates
\begin{align}
\label{myrealupdateA2}
\vct{z}_{\tau+1}=\mathcal{P}_{\mathcal{K}}\left(\vct{z}_\tau-\mu_\tau\nabla \mathcal{L}_A(\vct{z}_\tau)\right),
\end{align}
with $\mathcal{K}:=\{\vct{z}\in\R^n:\text{ }\mathcal{R}(\vct{z})\le \mathcal{R}(\vct{x})\}$ and $\nabla \mathcal{L}_A$ defined via \eqref{gengrad}. Also set the learning parameter sequence $\mu_0=0$ and $\mu_\tau=1$ for all $\tau=1,2,\ldots$. Furthermore, let $m_0=\mathcal{M}(\mathcal{R},\vct{x})$, defined by \ref{PTcurve}, be our lower bound on the number of measurements. Also assume
\begin{align}
\label{nummeaslin}
m>c m_0,
\end{align}
holds for a fixed numerical constant $c$. Then there is an event of probability at least $1-9e^{-\gamma m}$ such that on this event starting from any initial point obeying \eqref{initAPWF2} the update \eqref{myrealupdateA2} satisfy
\begin{align}
\label{ratemin}
\emph{dist}(\vct{z}_\tau,\vct{x})\le\left(\frac{2}{3}\right)^{\tau}\emph{dist}(\vct{z}_0,\vct{x}).
\end{align}
Here $\gamma$ is a fixed numerical constant.
\end{theorem}
The first interesting and perhaps surprising aspect of this result is its generality: it applies not only to convex regularization functions but also nonconvex ones! As we mentioned earlier the optimization problem in \eqref{opt} is not known to be tractable even for convex regularizers.  Despite the nonconvexity of both the objective and regularizer, the theorem above shows that with a near minimal number of measurements, projected gradient descent provably converges to the original signal $\vct{x}$ without getting trapped in any local optima.

The amplitude-based loss also has stronger sample complexity and computational complexity guarantees compared with the intensity-based version. Indeed, the required number of measurements improves upon the intensity-based loss by a logarithmic factor, achieving a near optimal sample complexity for this problem (up to a constant factor). Also, the convergence rate of the amplitude-based approach is now linear. Therefore, to achieve a relative error of $\epsilon$ the total number of iterations is on the order of $\mathcal{O}(\log(1/\epsilon))$. Thus the overall computational complexity is on the order of $\mathcal{O}\left(mn\log(1/\epsilon)\right)$ (in general the cost is the total number of iterations multiplied by the cost of applying the measurement matrix $\mtx{A}$ and its transpose). As a result, the computational complexity is also now optimal in terms of dependence on the matrix dimensions. Indeed, for a dense matrix even verifying that a good solution has been achieved requires one matrix-vector multiplication which takes $\mathcal{O}(mn)$ time.

We now pause to discuss the choice of the loss function. The theoretical results above suggests that the least squares loss on amplitude values is superior to one on intensity values in terms of both sample and computational complexity. Such improved performance has also been observed empirically for more realistic models in optics \cite{yeh2015experimental}. Indeed, \cite{yeh2015experimental} shows that not only the amplitude-based least squares has faster convergence rates but also is more robust to noise and model misspecification. However, we would like to point out that the least squares objective on intensity values does have certain advantages. For instance, it is possible to do exact line search (in closed form) on this objective. We have observed that this approach works rather well in some practical domains (e.g.~ptychography for chip imaging) without the need for any tuning as the step size in each iteration is calculated in closed form via exact line search. Therefore, we would like to caution against rushed judgments declaring one variant of Wirtinger Flow superior to another due to minor (e.g.~logarithmic) theoretical improvements in sample complexity and or computational complexity.\footnote{Unfortunately such premature declarations have become exceedingly common in recent literature.} We would like to emphasize that there is no ``best" or ``correct" loss function that works better than others for all application domains. Ultimately, the choice of the loss function is dictated by the statistics of the noise or misspecification present in a particular domain. 

\section{Discussions and prior art}\label{PART}
Phase retrieval is a century old problem and many heuristics have been developed for its solution. For a partial
    review of some of these heuristics as well as some recent
    theoretical advances in related problems we refer the reader to the overview articles/chapters \cite{shechtman2014phase, marchesini2007invited, jaganathan2015phase} and \cite[Part II]{soltanolkotabi2014algorithms} as well as \cite[Section 1.6]{candes2014phase}, and references therein such as \cite{cahillphase, ohlsson2011compressive,
      bandeira2013saving,
      waldspurger2012phase, jaganathan2012robust,gross2013partial,gross2014improved, leshem2017discrete, leshem2016direct}. There has also been a surge of activity surrounding nonconvex optimization problems in the last few years. While discussing all of these results is beyond the scope of this paper we shall briefly discuss some of the most relevant and recent literature in the coming paragraphs. We refer the reader to \cite{sun2015nonconvex} and references therein \cite{burer2003nonlinear, keshavan2009few, keshavan2009matrix, keshavan2012efficient, haeffele2014structured} for a more comprehensive review of such results. We also refer the reader to\cite{hand2016elementary, goldstein2016phasemax, bahmani2016phase} for recent algorithmic approaches based on linear programs and \cite{li2015phase} for characterizing large systems limits of dynamics of phase retrieval algorithms.
      
The Wirtinger Flow algorithms for solving quadratic systems of equations was introduce in \cite{WF}. \cite{WF} also provides a local convergence analysis when no prior structural assumption is available about the signal. The analysis of \cite{WF} was based on the so called regularity condition. This regularity condition and closely related notions have been utilized/generalized in a variety of interesting ways to provide rigorous convergence guarantees for related nonconvex problems arising in diverse applications ranging from matrix completion to dictionary learning and blind deconvolution \cite{balakrishnan2014statistical, sun14, desa14, arora2015simpleC, tu2015low, chen2015solving, cai2015optimal, zheng15, zhao15, chen15, zhang2016provable, sun2016geometric, zheng2016convergence, zhang2016reshaped, wang2016solving, Park:2016aa, Li:2016aa}. The intensity-based results presented in Section \ref{inten} are based on a generalization of the regularity condition in \cite{WF} so as to allow arbitrary convex constraints.

The second set of results we presented in this paper were based on least squares fit of the amplitudes. This objective function has been historically used in phase retrieval applications \cite{F82} and has close connections with the classical Fienup algorithm \cite[Chapter 13]{soltanolkotabi2014algorithms}. Focusing on more recent literature, \cite{yeh2015experimental} demonstrated the effectiveness of this approach in optical applications. More recently, a few interesting publications \cite{zhang2016reshaped, wang2016solving} study variants of this loss function and develop guarantees for its convergence. The analysis presented in both of these papers are also based on variants of the regularity condition of \cite{WF} and do not utilize any structural assumptions. In this paper we have analyzed the performance of the amplitude-based PWF with any constraint (convex or nonconvex). These results are based on a new approach to analyzing nonconvex optimization problems that differs from the regularity approach used in \cite{WF} and all of the papers mentioned above. Rather, this new technique follows a more direct route, utilizing/developing powerful concentration inequalities to directly show the error between the iterates and the structured signal decreases at each iteration.

A more recent line of research aims to provide a more general understanding of the geometric landscape of nonconvex optimization problems by showing that in many problems there are no spurious local minmizers and saddles points have favorable properties \cite{sun2016geometric, ge2016matrix, bhojanapalli2016global, boumal2016non}. A major advantage of such results is that they do not required specialized initializations in the sense that trust region-type algorithms or noisy stochastic methods are often guaranteed to converge from a random initialization and not just when an initial solution is available in a local neighborhood of the optimal solution. The disadvantage of such results is that the guaranteed rates of convergence of these approaches are either not linear/geometric or each iteration is very costly. These approaches also have slightly looser sample complexity bounds. Perhaps the most relevant result of this kind to this paper is the interesting work of Sun, Qu and Wright \cite{sun2016geometric} which studies the geometric landscape of the objective \eqref{optsec} in the absence of any regularizer. The authors also show that a certain trust region algorithm achieves a relative error of $\epsilon$ after $\mathcal{O}\left(n^7\log^7 n+\log \log \frac{1}{\epsilon}\right)$ as long as the number of samples exceeds $m\gtrsim n\log^3 n$. As mentioned previously, using different proof techniques in a companion paper we demonstrate a result of a similar flavor to \cite{sun2016geometric} for the constrained problem \eqref{optsec}. This results shows that with $m\gtrsim m_0\log n$ measurements all local optima are global and a second order scheme recovers the global optima (the unknown signal) in a polynomial number of iterations. 

We now pause to caution against erroneous miss-interpretations of the theoretical results discussed in the previous paragraph:
\begin{itemize}
\item There are no spurious local optima i.e.~all local optima are global in phase retrieval applications
\item Initialization is irrelevant in phase retrieval applications
\end{itemize}
The reason these conclusions are inaccurate are two-fold. First, while the results of the previous paragraph and Theorem \ref{Ithm} both require on the order of $m\gtrsim n\log n$ samples the multiplicative constants in these results tend to be drastically different in practice. Second, the measurement vectors occurring in practical domains are substantially more ill-conditioned than the Gaussian measurements studied in this paper. This further amplifies the gap between the sample complexity of local versus global results. Indeed, in many practical domains where phase retrieval is applied local optima are abound and a major source of algorithmic stagnation. Therefore, carefully crafted initialization schemes or regularization methods are crucial for the convergence of local search heuristics in many phase-less imaging domains.

We would also like to mention prior work on sparse phase retrieval. For generic measurements such as the Gaussian distribution studied in this paper, \cite{LiVoroninski2012} provides guarantees for the convex relaxation-based PhaseLift algorithm as long as the number of samples exceed $m\gtrsim s^2\log n$ where $s$ is the number of non-zeros in the sparse signal. The papers \cite{LiVoroninski2012, oymak2015simultaneously} showed that these results are essentially unimprovable when using simple SDP relaxations. More recently, interesting work by Cai, Li and Ma \cite{cai2015optimal} studies the performance of Wirtinger Flow based schemes for sparse phase retrieval problems. This result also requires $m\gtrsim s^2\log n$ measurements even when an initialization  obeying \eqref{initIPWF} is available. Therefore, this results also does not breakthrough the local $s^2$ barrier. More recently, there are a few publications aimed at going below $s^2$ measurements. These results differ from ours in that they are either applicable to specific designs which tailor the algorithm to the measurement process \cite{bahmani2015efficient, pedarsani2014phasecode, lee2016saffron, iwen2015fast, iwen2015robust} or require additional constraints on the coefficient of the sparse signal \cite{wang2016sparse, hand2016compressed, lee2013near}. In contrast to the above publications in this paper we have demonstrated that locally only $m\gtrsim s\log(n/s)$ samples suffice to recover \emph{any} $s$ sparse signal from generic quadratic measurements formally breaking through the $s^2$ barrier. Furthermore, our results applies to any regularizer (convex or nonconvex), allowing us to enforce various forms of prior knowledge in our reconstuction.

Finally, we would like to mention that there has also been some recent publications aimed at developing theoretical guarantees for more practical models. For instance, the papers \cite{candes2014phase, gross2014improved, salehi2017multiple, jaganathan2016reconstruction, jaganathan2016stft} develop theoretical guarantees for convex relaxation techniques for more realistic Fourier based models such as coded diffraction patterns and Ptychography. More recently, the papers \cite{iwen2016fast, iwen2016phase, jaganathan2013sparse, eldar2015sparse, bendory2016non} also develop some theoretical guarantees for faster but sometimes design-specific algorithms. Despite all of this interesting progress the known results for more realistic measurement models are far inferior to their Gaussian counterparts in terms of sample complexity, computational complexity or stability to noise. Closing these gaps is an interesting and important future direction.  

\section{Proofs}
In the Gaussian model the measurement vectors also obey $\twonorm{\vct{a}_r}\le \sqrt{6n}$ for all $r=1,2,\ldots,m$ with probability at least $1-me^{-1.5n}$. Thoughout the proofs, we assume we are on this event without explicitly mentioning it each time. Without loss of generality we will assume throughout the proofs that $\twonorm{\vct{x}}=1$. We remind the reader that throughout $\vct{x}$ is a solution to our quadratic equations, i.e.~obeys $\vct{y}=\abs{\mtx{A}\vct{x}}^2$ and that the sampling vectors are independent from $\vct{x}$. We also remind the reader that for a set $\mathcal{C}\subset\R^n$, $\omega(\mathcal{C})$ is the mean width of $\mathcal{C}$ per Definition \ref{Gausswidth}. Throughout, we use $\mathbb{S}^{n-1}/\mathcal{B}^n$ to denote the unit sphere/unit ball of $\R^n$. We first discuss some common background and results used for proving both theorems. Since the proof of the two theorems follow substantially different paths we dedicate a subsection to each: Section \ref{secIPWF} for proof of Theorem \ref{Ithm} and Section \ref{secAPWF} for proof of Theorem \ref{Athm}.
\subsection{Formulas for gradients, generalized gradients and their expected values}
As a reminder the intensity-based loss function is equal to
\begin{align*}
\mathcal{L}_I(\vct{z}):=\frac{1}{4m}\sum_{r=1}^m \left(\vct{y}_r-\abs{\vct{a}_r^*\vct{z}}^2\right)^2,
\end{align*}
and the gradient equal to
\begin{align*}
\nabla \mathcal{L}_I(\vct{z})=\frac{1}{m}\sum_{r=1}^m \left(\abs{\vct{a}_r^*\vct{z}}^2-\vct{y}_r\right)(\vct{a}_r^*\vct{z})\vct{a}_r.
\end{align*}
As a reminder the amplitude-based loss function is equal to
\begin{align*}
\mathcal{L}_A(\vct{z}):=\frac{1}{2m}\sum_{r=1}^m \left(\sqrt{\vct{y}_r}-\abs{\vct{a}_r^*\vct{z}}\right)^2,
\end{align*}
and the generalized gradient is equal to
\begin{align*}
\nabla \mathcal{L}_A(\vct{z})=\frac{1}{m}\sum_{r=1}^m\left(\abs{\vct{a}_r^*\vct{z}}-\sqrt{\vct{y}_r}\right)\sgn{\vct{a}_r^*\vct{z}}\vct{a}_r.
\end{align*}
\subsection{Concentration and bounds for stochastic processes}
In this section we gather some useful results on concentration of stochastic processes which will be crucial in our proofs. We begin with a lemma which is a direct consequence of Gordon's escape from the mesh lemma \cite{Gor} whose proof is deferred to Appendix \ref{pfgordontype}. 
\begin{lemma}\label{gordontype}Assume $\mathcal{C}\subset\R^n$ is a cone and $\mathbb{S}^{n-1}$ is the unit sphere of $\R^n$. Also assume that
\begin{align*}
m\ge \max\left(20\frac{\omega^2(\mathcal{C}\cap\mathbb{S}^{n-1})}{\delta^2},\frac{1}{2\delta}-1\right),
\end{align*}
for a fixed numerical constant $c$. Then for all $\vct{h}\in\mathcal{C}$
\begin{align*}
\abs{\frac{1}{m}\sum_{r=1}^m(\vct{a}_r^*\vct{h})^2-\twonorm{\vct{h}}^2}\le \delta\twonorm{\vct{h}}^2,
\end{align*}
holds with probability at least $1-2e^{-\frac{\delta^2}{360}m}$. 
\end{lemma}
We also need a generalization of the above lemma stated below and proved in Appendix \ref{appendix2}.
\begin{lemma}\label{gordontypenonsym}Assume $\mathcal{C}\subset\R^n$ is a cone (not necessarily convex) and $\mathbb{S}^{n-1}$ is the unit sphere of $\R^n$. Also assume that
\begin{align*}
m\ge \max\left(80\frac{\omega^2(\mathcal{C}\cap\mathbb{S}^{n-1})}{\delta^2},\frac{2}{\delta}-1\right),
\end{align*}
for a fixed numerical constant $c$. Then for all $\vct{u},\vct{h}\in\mathcal{C}$
\begin{align*}
\abs{\frac{1}{m}\sum_{r=1}^m(\vct{a}_r^*\vct{u})(\vct{a}_r^*\vct{h})-\vct{u}^*\vct{h}}\le \delta\twonorm{\vct{u}}\twonorm{\vct{h}},
\end{align*}
holds with probability at least $1-6e^{-\frac{\delta^2}{1440}m}$. 
\end{lemma}
We next state a generalization of Gordon's escape through the mesh lemma, whose proof appears in Appendix \ref{pfgordontype2}.
\begin{lemma}\label{GTtypelem} Let $\vct{d}\in\R^n$ be fixed vector with nonzero entries and construct the diagonal matrix $\mtx{D}=\text{diag}(\vct{d})$. Also, let $\mtx{A}\in\R^{m\times n}$ have i.i.d.~$\mathcal{N}(0,1)$ entries. Furthermore, assume $\mathcal{T}\subset\R^n$ and define
\begin{align*}
b_m(\vct{d})=\E[\twonorm{\mtx{D}\vct{g}}],
\end{align*}
where $\vct{g}\in\R^m$ is distributed as $\mathcal{N}(\vct{0},\mtx{I}_m)$. Define
\begin{align*}
\sigma(\mathcal{T}):=\underset{\vct{v}\in\mathcal{T}}{\max}\text{ }\twonorm{\vct{v}},
\end{align*}
then for all $\vct{u}\in\mathcal{T}$ 
\begin{align*}
\abs{\twonorm{\mtx{D}\mtx{A}\vct{u}}-b_m(\vct{d})\twonorm{\vct{u}}}\le \infnorm{\vct{d}}\omega(\mathcal{T})+\eta,
\end{align*}
holds with probability at least 
\begin{align*}
1-6e^{-\frac{\eta^2}{8\infnorm{\vct{d}}^2\sigma^2(\mathcal{T})}}.
\end{align*}
\end{lemma}
The previous lemma leads to the following Corollary. We skip the proof as it is identical to how Lemma \ref{gordontype} is derived from Gordon's lemma (See Section \ref{pfgordontype} for details).
\begin{corollary}\label{coroG}
Let $\vct{d}\in\R^n$ be fixed vector with nonzero entries and assume $\mathcal{T}\subset\mathcal{B}^n$. Furthermore, assume
\begin{align*}
\left(\sum_{r=1}^m d_r^2\right)\ge \max\left(20\infnorm{\vct{d}}^2\frac{\omega^2(\mathcal{T})}{\delta^2},\frac{3}{2\delta}-1\right).
\end{align*}
Then for all $\vct{u}\in\mathcal{T}$, 
\begin{align*}
\abs{\frac{\sum_{r=1}^md_r^2(\vct{a}_r^*\vct{u})^2}{\sum_{r=1}^m d_r^2}-\twonorm{\vct{u}}^2}\le \delta,
\end{align*}
holds with probability at least $1-6e^{-\frac{\delta^2}{1440}\left(\sum_{r=1}^m d_r^2\right)}$.
\end{corollary}
The above generalization of Gordon's lemma together with its corollary will be very useful in our proofs in particular it allows us to prove the following key result whose proof is also deferred to Appendix \ref{pfgordontype3}.
\begin{lemma}\label{GordonExtra}
Assume $\mathcal{C}\subset\R^n$ is a cone and $\mathbb{S}^{n-1}$ is the unit sphere of $\R^n$. Furthermore, let $\vct{x}\in\R^n$ be a fixed vector. Also assume that
\begin{align*}
m\ge 1600\cdot\max\left(\frac{\omega^2(\mathcal{C}\cap \mathbb{S}^{n-1})}{\delta^2}\log n,\frac{1}{\delta^2}\right).
\end{align*}
Then for all $\vct{h}\in\mathcal{C}$
\begin{align*}
\frac{1}{m}\sum_{r=1}^m(\vct{a}_r^*\vct{h})^2(\vct{a}_r^*\vct{x})^2-\left(\twonorm{\vct{h}}^2\twonorm{\vct{x}}^2+2(\vct{h}^*\vct{x})^2\right)\le \delta\twonorm{\vct{h}}^2\twonorm{\vct{x}}^2,
\end{align*}
holds with probability at least $1-2/m-1/n-e^{\gamma_1 m}-7e^{-\gamma_2\delta^2m}$ with $\gamma_1$ and $\gamma_2$ fixed numerical constants.
\end{lemma}
We also need the following important lemma. The proof of this lemma is based on the paper \cite{mendelson2012oracle}. Please also see \cite{eldar2014phase} for related calculations. We defer the proof to Appendix \ref{ProofRIP1}. 
\begin{lemma}\label{RIP1} Assume $\mathcal{C},\mathcal{C}'\subset \R^n$ are cones and $\mathbb{S}^{n-1}$ is the unit sphere of $\R^n$. Also assume
\begin{align*}
m\ge c\cdot\max\left( \omega^2(\mathcal{C}\cap \mathbb{S}^{n-1}),\omega^2(\mathcal{C}'\cap\mathbb{S}^{n-1})\right),
\end{align*}
for a fixed numerical constant $c$. Then for any $\vct{u}\in\mathcal{C}$ and $\vct{v}\in\mathcal{C}'$ 
\begin{align*}
\abs{\frac{1}{m}\sum_{r=1}^m\abs{\vct{u}^*\vct{a}_r\vct{a}_r^*\vct{v}}-\E[\abs{\vct{u}^*\vct{a}\vct{a}^*\vct{v}}]}\le\delta\twonorm{\vct{u}}\twonorm{\vct{v}},
\end{align*}
holds with probability at least $1-2e^{-\gamma\delta m}$ where $\gamma$ is a fixed numerical constant. Here, $\vct{a}\in\R^n$ is distributed as $\mathcal{N}(\vct{0},\mtx{I})$.
\end{lemma}
We also state a simple generalization of Lemma \ref{RIP1} above. This lemma has a near identical proof. We skip details for brevity.
\begin{lemma}\label{RIP2} Assume $\mathcal{C},\mathcal{D}\subset \R^n$ are sets with diameters bounded by fixed numerical constants. Also assume
\begin{align*}
m\ge c\cdot\max\left( \omega^2(\mathcal{C}),\omega^2(\mathcal{D})\right),
\end{align*}
for a fixed numerical constant $c$. Then for any $\vct{u}\in\mathcal{C}$ and $\vct{v}\in\mathcal{D}$ 
\begin{align*}
\abs{\frac{1}{m}\sum_{r=1}^m\abs{\vct{u}^*\vct{a}_r\vct{a}_r^*\vct{v}}-\E[\abs{\vct{u}^*\vct{a}\vct{a}^*\vct{v}}]}\le\delta
\end{align*}
holds with probability at least $1-2e^{-\gamma\delta m}$ where $\gamma$ is a fixed numerical constant. Here, $\vct{a}\in\R^n$ is distributed as $\mathcal{N}(\vct{0},\mtx{I})$.
\end{lemma}
Finally, we also need the following lemma with the proof appearing in Appendix \ref{pfgordontype6}. 
\begin{lemma}\label{expRIP1}
For any $\vct{u},\vct{v}\in\R^n$ define $\theta=\cos^{-1}\left(\frac{\vct{u}^T\vct{v}}{\twonorm{\vct{u}}\twonorm{\vct{v}}}\right)$. Then we have
\begin{align*}
\E[\abs{\vct{u}^*\vct{a}\vct{a}^*\vct{v}}]= \frac{2}{\pi}\twonorm{\vct{u}}\twonorm{\vct{v}}\left(\sin(\theta)+\cos(\theta)\left(\frac{\pi}{2}-\theta\right)\right)\ge \frac{2}{\pi}\twonorm{\vct{u}}\twonorm{\vct{v}}.
\end{align*}
\end{lemma}
\subsection{Cone and projection identities}
In this section we will gather a few results regarding higher dimensional cones and projections that are used throughout the proofs. These results are directly adapted from \cite[6.1]{oymak2015sharp}. We begin with a result about projections onto sets. The first part concerning projections onto convex sets is the well known contractivity result regarding convex projections. 
\begin{lemma}\label{projcvxthm} Assume $\mathcal{K}\in\R^n$ is a closed set and $\vct{v}\in\R^n$. Then if $\mathcal{K}$ is convex for every $\vct{u}\in\mathcal{K}$ we have
\begin{align}\label{projcvxthm1}
\twonorm{\mathcal{P}_{\mathcal{K}}(\vct{v})-\vct{u}}\le\twonorm{\vct{v}-\vct{u}}.
\end{align}
Furthermore, for any closed set $\mathcal{K}$ (not necessarily convex) and for every $\vct{u}\in\mathcal{K}$ we have
\begin{align}\label{projncvxthm1}
\twonorm{\mathcal{P}_{\mathcal{K}}(\vct{v})-\vct{u}}\le2\twonorm{\vct{v}-\vct{u}}.
\end{align}
\end{lemma}
\begin{proof}
Equation \eqref{projcvxthm1} is well known. We shall prove the second result. To this aim note that by definition of projection onto a set we have
\begin{align}
\label{mytempineq}
\twonorm{\vct{v}-\mathcal{P}_{\mathcal{K}}(\vct{v})}^2\le \twonorm{\vct{v}-\vct{u}}^2.
\end{align}
Also note that
\begin{align*}
\twonorm{\vct{v}-\mathcal{P}_{\mathcal{K}}(\vct{v})}^2=&\twonorm{(\vct{v}-\vct{u})-\left(\mathcal{P}_{\mathcal{K}}(\vct{v})-\vct{u}\right)}^2,\\
=&\twonorm{\vct{v}-\vct{u}}^2+\twonorm{\mathcal{P}_{\mathcal{K}}(\vct{v})-\vct{u}}^2-2\langle \mathcal{P}_{\mathcal{K}}(\vct{v})-\vct{u},\vct{v}-\vct{u}\rangle.
\end{align*}
Combining the latter inequality with \eqref{mytempineq} and using the  Cauchy-Schwarz inequality we have
\begin{align*}
\twonorm{\mathcal{P}_{\mathcal{K}}(\vct{v})-\vct{u}}^2=&\twonorm{\vct{v}-\mathcal{P}_{\mathcal{K}}(\vct{v})}^2-\twonorm{\vct{v}-\vct{u}}^2+2\langle \mathcal{P}_{\mathcal{K}}(\vct{v})-\vct{u},\vct{v}-\vct{u}\rangle,\\
\le&2\langle \mathcal{P}_{\mathcal{K}}(\vct{v})-\vct{u},\vct{v}-\vct{u}\rangle,\\
\le&2\twonorm{\mathcal{P}_{\mathcal{K}}(\vct{v})-\vct{u}}\twonorm{\vct{v}-\vct{u}}.
\end{align*}
Dividing both sides of the above inequality by $\twonorm{\mathcal{P}_{\mathcal{K}}(\vct{v})-\vct{u}}$ concludes the proof.
\end{proof}
We now state a result concerning projection onto cones.
\begin{lemma}\label{firstconelem} Let $\mathcal{C}\subset\R^n$ be a closed cone and $\vb\in\R^n$. The followings two identities hold
\begin{align}
\tn{\vb}^2=&\tn{\vb-\mathcal{P}_{\mathcal{C}}(\vb)}^2+\tn{\mathcal{P}_\Cc(\vb)}^2\label{orthogonal},\\
\tn{\mathcal{P}_\Cc(\vb)}=&\sup_{\ub\in\Cc\cap\Bc^{n}} \ub^*\vb\label{lem:identity}.
\end{align}
\end{lemma}
%
The following lemma is straightforward and follows from the fact that translation preserves distances.
\begin{lemma} \label{simplem}Suppose $\mathcal{K}\subset\R^n$ is a closed set. The projection onto $\mathcal{K}$ obeys
\begin{align*}
\Pc_{\mathcal{K}}(\x+\vb)-\vct{x}=\Pc_{\mathcal{K}-\{\x\}}(\vb).
\end{align*}
\end{lemma}
The next lemma compares the length of a projection onto a set to the length of projection onto the conic approximation of the set. 
\begin{lemma} [Comparison of projections] \label{prop compare} Let $\mathcal{D}$ be a closed and nonempty set that contains $\vct{0}$. Let $\Cc$ be a nonempty and closed cone containing $\mathcal{D}$ ($\mathcal{D}\subset\mathcal{C}$). Then for all $\vct{v}\in\R^n$,
\begin{align}
\label{firstsetcomp}
\tn{\Pc_\Dc(\vb)}\le 2\tn{\Pc_\Cc(\vb)}
\end{align}
Furthermore, assume $\Dc$ is a convex set. Then for all $\vb\in\R^n$, 
\begin{align}
\label{cvxsetcomp}
\tn{\Pc_\Dc(\vb)}\le \tn{\Pc_\Cc(\vb)}.
\end{align}
\end{lemma}
\subsection{Convergence analysis for intensity-based Wirtinger Flows}
\label{secIPWF}
In this section we shall prove Theorem \ref{Ithm}. The proof of this result is based on an extension of the framework developed in \cite{WF}. Therefore, the outline of our exposition closely follows that of \cite{WF}. Section \ref{genconv} discusses our general convergence analysis and shows that it follows from a certain \emph{Regularity Condition (RC)}. In this section we also show that the regularity condition can be proven by showing two sufficient \emph{Local Curvature} and \emph{Local Smoothness} conditions denoted by LCC and LSC. We then prove the Local Curvature condition in Section \ref{pfLCC} and the Local Smoothness condition in Section \ref{pfLSC}.

\subsubsection{General convergence analysis}
\label{genconv}
Note that \eqref{initIPWF} guarantees that either $\twonorm{\vct{z}_0-\vct{x}}$ or $\twonorm{\vct{z}_0+\vct{x}}$ is small. Throughout the proof without loss of generality we assume  $\twonorm{\vct{z}_0-\vct{x}}$ is the smaller one. To introduce our general convergence analysis we begin by defining 
\begin{align*}
E(\epsilon)=\big\{\vct{z}\in\R^n:\mathcal{R}(\vct{z})\le \mathcal{R}(\vct{x}),\text{ }\twonorm{\vct{z}-\vct{x}}\le\epsilon\big\}.
\end{align*}
Note that when condition \eqref{initIPWF} holds the next iterate $\vct{z}_1$ obeys $\vct{z}_1\in E(\epsilon)$ with $\epsilon=1/8$. The reason is that when the regularizer is convex so is the set $\mathcal{K}$ and by contractivity of projection onto convex sets (Lemma \ref{projcvxthm}) we also have
\begin{align*}
\twonorm{\vct{z}_1-\vct{x}}=\twonorm{\mathcal{P}_{\mathcal{K}}\left(\vct{z}_0\right)-\vct{x}}\le \twonorm{\vct{z}_0-\vct{x}}\le \epsilon.
\end{align*}
We will assume that the function $\mathcal{L}_{I}$ satisfies a regularity condition
on $E(\epsilon)$, which essentially states that the gradient of the
function is well-behaved.

\begin{condition}[Regularity Condition] We say that the function $\mathcal{L}_I$ satisfies the regularity condition or $RC(\alpha,\beta,\epsilon)$ if for all vectors $\vct{z}\in E(\epsilon)$ we have
\begin{align}
\label{regcond}
\big\langle \nabla \mathcal{L}_I(\vct{z}),\vct{z}-\vct{x}\big\rangle  \ge \frac{1}{\alpha}\twonorm{\vct{z}-\vct{x}}^2+\frac{1}{\beta}\twonorm{\nabla \mathcal{L}_I(\vct{z})}^2.
\end{align}
\end{condition}

In the lemma below we show that as long as the regularity condition
holds on $E(\epsilon)$ then Projected Wirtinger Flow starting from an initial
solution in $E(\epsilon)$ converges to a global optimizer at a
geometric rate. Subsequent sections shall establish that this property
holds.

\begin{lemma}
\label{convergence}
Assume that $\mathcal{L}_I$ obeys RC$(\alpha,\beta,\epsilon)$ for all $\vct{z} \in
E(\epsilon)$. Furthermore, suppose $\vct{z}_1 \in E(\epsilon)$, and assume
$0<\mu\le {2}/{\beta}$. Consider the following update
\[
  \vct{z}_{\tau+1}=\mathcal{P}_{\mathcal{K}}\left(\vct{z}_\tau-\mu\nabla f(\vct{z}_\tau)\right).
\]
Then for all $\tau$ we have $\vct{z}_\tau\in E(\epsilon)$ and
\begin{align*}
\twonorm{\vct{z}_\tau-\vct{x}}^2\le \left(1-\frac{2\mu}{\alpha}\right)^\tau \twonorm{\vct{z}_0-\vct{x}}^2.
  \end{align*}
\end{lemma}

\begin{proof}
  The proof is similar to a related proof in the Wirtinger Flow paper \cite{WF}. We prove that if $\vct{z}\in E(\epsilon)$ then for all $0<\mu\le
  {2}/{\beta}$
\[
  \vct{z}_+=\vct{z}-\mu\nabla f(\vct{z})
\]
obeys
\begin{align}
\label{iter}
\twonorm{\vct{z}_+-\vct{x}}^2\le\left(1-\frac{2\mu}{\alpha}\right)\twonorm{\vct{z}-\vct{x}}^2.
\end{align}
The latter implies that if $\twonorm{\vct{z}-\vct{x}}\le \epsilon$, then $\twonorm{\vct{z}_+-\vct{x}}\le \epsilon$. Combining the latter with the fact that projection onto convex sets are contractive (Lemma \ref{projcvxthm}) we conclude that
\begin{align}
\label{secondind}
\twonorm{\mathcal{P}_{\mathcal{K}}(\vct{z}_+)-\vct{x}}=\twonorm{\mathcal{P}_{\mathcal{K}}(\vct{z}_+)-\mathcal{P}(\vct{x})}\le\twonorm{\vct{z}_+-\vct{x}}\le\twonorm{\vct{z}-\vct{x}}\le \epsilon.
\end{align}
Also by the definition of $\mathcal{P}_{\mathcal{K}}$ we have $\mathcal{R}(\mathcal{P}_{\mathcal{K}}(\vct{z}_+))\le \mathcal{R}(\vct{x})$. Therefore, if $\vct{z}\in E(\epsilon)$ then we also have $\mathcal{P}_{\mathcal{K}}(\vct{z}_+)\in
E(\epsilon)$. The lemma follows by inductively applying
\eqref{iter} and \eqref{secondind}. Now let us demonstrate how \eqref{iter} follows from simple algebraic manipulations together with the
regularity condition \eqref{regcond}. To this aim note that
\begin{align*}
\twonorm{\vct{z}_{+}-\vct{x}}^2=&\twonorm{\vct{z}-\vct{x}-\mu\nabla \mathcal{L}_I(\vct{z})}^2\nonumber\\
=&\twonorm{\vct{z}-\vct{x}}^2-2\mu \big\langle \nabla \mathcal{L}_I(\vct{z}),\left(\vct{z}-\vct{x}\right)\big\rangle+\mu^2\twonorm{\nabla \mathcal{L}_I(\vct{z})}^2
\\
\le& \twonorm{\vct{z}-\vct{x}}^2 -2\mu\left(\frac{1}{\alpha}\twonorm{\vct{z}-\vct{x}}^2+\frac{1}{\beta}\twonorm{\nabla \mathcal{L}_I(\vct{z})}^2\right)+\mu^2\twonorm{\nabla \mathcal{L}_I(\vct{z})}^2\\
=&\left(1-\frac{2\mu}{\alpha}\right)\twonorm{\vct{z}-\vct{x}}^2+\mu\left(\mu-\frac{2}{\beta}\right)\twonorm{\nabla \mathcal{L}_I(\vct{z})}^2\\
\le&\left(1-\frac{2\mu}{\alpha}\right)\twonorm{\vct{z}-\vct{x}}^2,
\end{align*}
where the last line follows from $\mu\le {2}/{\beta}$. This concludes the proof.
\end{proof}

For any $\vct{z}\in E(\epsilon)$, we need to show that
\begin{align}
\label{tempRG}
\big\langle \nabla \mathcal{L}_I(\vct{z}),\vct{z}-\vct{x}\big\rangle \ge \frac{1}{\alpha}\twonorm{\vct{z}-\vct{x}}^2+\frac{1}{\beta}\twonorm{\nabla \mathcal{L}_I(\vct{z})}^2.
\end{align} 
We prove that \eqref{tempRG} holds with $\epsilon=\frac{1}{8}$ by establishing that our gradient
satisfies the local smoothness and local curvature conditions defined
below. Combining both these two properties gives \eqref{tempRG}.

\begin{condition}[Local Curvature Condition] We say that the function
  $\mathcal{L}_I$ satisfies the local curvature condition or
  $LCC(\alpha,\epsilon,\delta)$ if for all vectors $\vct{z}\in
  E(\epsilon)$,
\begin{align}
\label{LCC}
\big\langle \nabla \mathcal{L}_I(\vct{z}),\vct{z}-\vct{x}\big\rangle  \ge \left(\frac{1}{\alpha}+\lambda\right)\twonorm{\vct{z}-\vct{x}}^2+\frac{\gamma}{m}\sum_{r=1}^m\abs{\vct{a}_r^*(\vct{z}-\vct{x})}^4.
\end{align}

\end{condition}
This condition essentially states that the function curves
sufficiently upwards (along most directions) near the curve of global
optimizers.
\begin{condition}[Local Smoothness Condition] We say that the function $\mathcal{L}_I$ satisfies the local smoothness condition or $LSC(\beta,\epsilon,\delta)$ if for all vectors $\vct{z}\in E(\epsilon)$ we have
\begin{align}
\label{LSC}
\twonorm{\nabla \mathcal{L}_I(\vct{z})}^2\le\beta\left(\lambda\twonorm{\vct{z}-\vct{x}}^2+\frac{\gamma}{m}\sum_{r=1}^m\abs{\vct{a}_r^*(\vct{z}-\vct{x})}^4\right).
\end{align}
\end{condition}
This condition essentially states that the gradient of the function is
well behaved (the function does not vary too much) near the curve of
global optimizers.

\subsubsection{Proof of the local curvature condition}
\label{pfLCC}
For any $\vct{z}\in E(\epsilon)$, we want to prove the local curvature
condition \eqref{LCC}. Recall that
\begin{align*}
\nabla \mathcal{L}_I(\vct{z})=\frac{1}{m}\sum_{r=1}^m\left(\abs{\langle\vct{a}_r,\vct{z}\rangle}^2-y_r\right)(\vct{a}_r\vct{a}_r^*)\vct{z},
\end{align*} 
and define $\vct{h}:=\vct{z}-\vct{x}$. To establish \eqref{LCC} it suffices to prove that 
\begin{multline}
\label{tempha}
\frac{1}{m}\sum_{r=1}^m \left(2(\vct{h}^*\vct{a}_r\vct{a}_r^*\vct{x})^2+3(\vct{h}^*\vct{a}_r\vct{a}_r^*\vct{x})|\vct{a}_r^*\vct{h}|^2+|\vct{a}_r^*\vct{h}|^4\right)-\left(\frac{\gamma}{m}\sum_{r=1}^m\abs{\vct{a}_r^*\vct{h}}^4\right)\geq \left(\frac{1}{\alpha}+\lambda\right)
\twonorm{\vct{h}}^2,
\end{multline}
holds for all $\vct{h}$ satisfying $\|\vct{h}\|_2\leq \epsilon$. Equivalently, we only need to prove that for all $\vct{h}$ satisfying $\twonorm{\vct{h}}\le\epsilon$ we have
\begin{align}
\label{temphb}
\frac{1}{m}\sum_{r=1}^m \left(2(\vct{h}^*\vct{a}_r\vct{a}_r^*\vct{x})^2+3(\vct{h}^*\vct{a}_r\vct{a}_r^*\vct{x})|\vct{a}_r^*\vct{h}|^2+(1-\gamma)|\vct{a}_r^*\vct{h}|^4\right)
\geq \left(\frac{1}{\alpha}+\lambda\right)
\twonorm{\vct{h}}^2.
\end{align}
Define the following cone which is the cone of descent of $\mathcal{R}$ at $\vct{x}$
\begin{align*}
\mathcal{C}_{\mathcal{R}}(\vct{x})=\{\vct{h}:\text{ }\mathcal{R}(\vct{x}+c\vct{h})\le \mathcal{R}(\vct{x})\text{ for some }c>0\}.
\end{align*}
Now note that since $\vct{h}=\vct{z}-\vct{x}$ and $\mathcal{R}(\vct{z})\le \mathcal{R}(\vct{x})$, therefore $\vct{h}\in\mathcal{C}_{\mathcal{R}}(\vct{x})$. Note that by Lemma \ref{GordonExtra} as long as we have
\begin{align*}
m\ge 1600\cdot\max\left(\frac{\omega^2(\mathcal{C}_{\mathcal{R}}\cap \mathbb{S}^{n-1})}{\delta^2}\log n,\frac{1}{\delta^2}\right)=1600\cdot\max\left(\frac{m_0}{\delta^2}\log n,\frac{1}{\delta^2}\right),
\end{align*}
then
\begin{align*}
\frac{1}{m}\sum_{r=1}^m(\vct{a}_r^*\vct{h})^2(\vct{a}_r^*\vct{x})^2\le (1+\delta)\twonorm{\vct{h}}^2+2(\vct{h}^*\vct{x})^2,
\end{align*}
holds with probability at least $1-2/m-1/n-e^{\gamma_1 m}-7e^{-\gamma_2\delta^2m}$ with $\gamma_1$ and $\gamma_2$ fixed numerical constants. Therefore, to establish the local curvature condition \eqref{LCC} it suffices to show
\begin{align}
\label{temphc}
\frac{1}{m}\sum_{r=1}^m \left((\eta+2)(\vct{h}^*\vct{a}_r\vct{a}_r^*\vct{x})^2+3(\vct{h}^*\vct{a}_r\vct{a}_r^*\vct{x})|\vct{a}_r^*\vct{h}|^2+(1-\gamma)|\vct{a}_r^*\vct{h}|^4\right)
\geq \left(\frac{1}{\alpha}+\lambda+\eta(1+\delta)\right)\twonorm{\vct{h}}^2+2\eta(\vct{h}^*\vct{x})^2.
\end{align}
We pick $\eta$ such that $2\sqrt{\eta+2}\sqrt{1-\gamma}=3$. This is equivalent to establishing that
\begin{align}
\label{eq:reduced_assumption1_altt}
\frac{1}{m}\sum_{r=1}^m& \left(\frac{3}{2}\frac{1}{\sqrt{1-\gamma}}(\vct{h}^*\vct{a}_r\vct{a}_r^*\vct{x})+\sqrt{1-\gamma}\abs{\vct{a}_r^*\vct{h}}^2\right)^2\nonumber\\
&\ge\left(\frac{1}{\alpha}+\lambda-2(1+\delta)+\frac{9}{4(1-\gamma)}(1+\delta)\right)\twonorm{\vct{h}}^2+2\left(\frac{9}{4(1-\gamma)}-2\right)(\vct{h}^*\vct{x})^2.
\end{align}
We note that
\begin{align*}
\frac{3}{2}\frac{1}{\sqrt{1-\gamma}}(\vct{h}^*\vct{a}_r\vct{a}_r^*\vct{x})+\sqrt{1-\gamma}\abs{\vct{a}_r^*\vct{h}}^2=\left(\frac{3}{2}\frac{1}{\sqrt{1-\gamma}}\vct{h}\right)^*\vct{a}_r\vct{a}_r^*\left(\vct{x}+\frac{2}{3}(1-\gamma)\vct{h}\right).
\end{align*}
%
Therefore, it suffices to prove
\begin{align}
\label{eq:reduced_assumption1_alttt}
\frac{1}{m}\sum_{r=1}^m \left(\vct{h}^*\vct{a}_r\vct{a}_r^*\left(\vct{x}+\frac{2}{3}(1-\gamma)\vct{h}\right)\right)^2\ge\left(\frac{4}{9}(1-\gamma)\left(\frac{1}{\alpha}+\lambda-2(1+\delta)\right)+(1+\delta)\right)\twonorm{\vct{h}}^2+\frac{2}{9}(1+8\gamma)(\vct{h}^*\vct{x})^2.
\end{align}
Noting that
\begin{align*}
\frac{1}{m}\sum_{r=1}^m \left(\vct{h}^*\vct{a}_r\vct{a}_r^*\left(\vct{x}+\frac{2}{3}(1-\gamma)\vct{h}\right)\right)^2\ge&\left(\frac{1}{m}\sum_{r=1}^m \abs{\vct{h}^*\vct{a}_r\vct{a}_r^*\left(\vct{x}+\frac{2}{3}(1-\gamma)\vct{h}\right)}\right)^2,
\end{align*}
it suffices to prove
\begin{align}
\label{eq:reduced_assumption1_alt}
\frac{1}{m}\sum_{r=1}^m \abs{\vct{h}^*\vct{a}_r\vct{a}_r^*\left(\vct{x}+\frac{2}{3}(1-\gamma)\vct{h}\right)}\ge\sqrt{\left(\frac{4}{9}(1-\gamma)\left(\frac{1}{\alpha}+\lambda-2-2\delta\right)+(1+\delta)\right)\twonorm{\vct{h}}^2+\frac{2}{9}(1+8\gamma)(\vct{h}^*\vct{x})^2}.
\end{align}
To establish \eqref{eq:reduced_assumption1_alt} we shall utilize Lemma \ref{RIP1}. To this aim note that since $\vct{h}=\vct{z}-\vct{x}$ and $\mathcal{R}(\vct{z})\le \mathcal{R}(\vct{x})$, therefore $\vct{h}\in\mathcal{C}_{\mathcal{R}}(\vct{x})$. Now define the set
\begin{align*}
\mathcal{T}=\Bigg\{\frac{\frac{(\gamma+\frac{1}{2})}{(1-\gamma)}\vct{x}+\vct{z}}{\twonorm{\frac{(\gamma+\frac{1}{2})}{(1-\gamma)}\vct{x}+\vct{z}}}:\quad \mathcal{R}(\vct{z})\le \mathcal{R}(\vct{x})\quad\text{and}\quad \twonorm{\vct{z}-\vct{x}}\le\epsilon\Bigg\},
\end{align*}
and set $\mathcal{C}'=$cone$(\mathcal{T})$. Note that
\begin{align*}
\frac{3}{2}\frac{1}{(1-\gamma)}\left(\vct{x}+\frac{2}{3}(1-\gamma)\vct{h}\right)=\frac{(\gamma+\frac{1}{2})}{(1-\gamma)}\vct{x}+\vct{z}\in\mathcal{C}'\quad\Rightarrow\quad\left(\vct{x}+\frac{2}{3}(1-\gamma)\vct{h}\right)\in\mathcal{C}'.
\end{align*}
Also note that for all $\vct{z}\in E(\epsilon)$ with $\epsilon<\frac{3}{2}\frac{1}{1-\gamma}$
\begin{align}
\label{mytemp}
\twonorm{\frac{(\gamma+\frac{1}{2})}{(1-\gamma)}\vct{x}+\vct{z}}=\twonorm{\frac{3}{2}\frac{1}{(1-\gamma)}\vct{x}+\vct{z}-\vct{x}}\ge\frac{3}{2}\frac{1}{(1-\gamma)}-\epsilon> 0.
\end{align}
Similarly, for all $\vct{z}\in E(\epsilon)$
\begin{align}
\label{mytemp2}
\twonorm{\frac{(\gamma+\frac{1}{2})}{(1-\gamma)}\vct{x}+\vct{z}}=\twonorm{\frac{3}{2}\frac{1}{(1-\gamma)}\vct{x}+\vct{z}-\vct{x}}\le\frac{3}{2}\frac{1}{(1-\gamma)}+\epsilon.
\end{align}
Define
\begin{align*}
\mathcal{T}'=\Big\{\frac{(\gamma+\frac{1}{2})}{(1-\gamma)}\vct{x}+\vct{z}:\text{ }\mathcal{R}(\vct{z})\le \mathcal{R}(\vct{x})\quad\text{and}\quad\twonorm{\vct{z}-\vct{x}}\le\epsilon\Big\}.
\end{align*}
Now set $\vct{v}=\underset{\vct{u}\in\mathcal{T}}{\arg\max}$ $\vct{a}^*\vct{u}$. By definiton of $\mathcal{T}$, $\vct{v}$ is of the form $\vct{v}=\frac{\frac{(\gamma+\frac{1}{2})}{(1-\gamma)}\vct{x}+\vct{z}_{\vct{v}}}{\twonorm{\frac{(\gamma+\frac{1}{2})}{(1-\gamma)}\vct{x}+\vct{z}_{\vct{v}}}}$ for some $\vct{z}_{\vct{v}}\in E(\epsilon)$. If $\vct{a}^*\vct{v}\ge 0$, using \eqref{mytemp} we have
\begin{align}
\label{mytemp3}
\vct{a}^*\vct{v}=\vct{a}^*\left(\frac{\frac{(\gamma+\frac{1}{2})}{(1-\gamma)}\vct{x}+\vct{z}_{\vct{v}}}{\twonorm{\frac{(\gamma+\frac{1}{2})}{(1-\gamma)}\vct{x}+\vct{z}_{\vct{v}}}}\right)\le \frac{1}{\left(\frac{3}{2}\frac{1}{(1-\gamma)}-\epsilon\right)}\vct{a}^*\left(\frac{(\gamma+\frac{1}{2})}{(1-\gamma)}\vct{x}+\vct{z}_{\vct{v}}\right)\le\frac{1}{\left(\frac{3}{2}\frac{1}{(1-\gamma)}-\epsilon\right)}\left(\underset{\vct{u}\in\mathcal{T}'}{\sup}\text{ }\vct{a}^*\vct{u}\right).
\end{align}
On the other hand if $\vct{a}^*\vct{v}< 0$, using \eqref{mytemp2} we have
\begin{align}
\label{mytemp4}
\vct{a}^*\vct{v}=\vct{a}^*\left(\frac{\frac{(\gamma+\frac{1}{2})}{(1-\gamma)}\vct{x}+\vct{z}_{\vct{v}}}{\twonorm{\frac{(\gamma+\frac{1}{2})}{(1-\gamma)}\vct{x}+\vct{z}_{\vct{v}}}}\right)\le \frac{1}{\left(\frac{3}{2}\frac{1}{(1-\gamma)}+\epsilon\right)}\vct{a}^*\left(\frac{(\gamma+\frac{1}{2})}{(1-\gamma)}\vct{x}+\vct{z}_{\vct{v}}\right)\le\frac{1}{\left(\frac{3}{2}\frac{1}{(1-\gamma)}+\epsilon\right)}\left(\underset{\vct{u}\in\mathcal{T}'}{\sup}\text{ }\vct{a}^*\vct{u}\right).
\end{align}
Inequalities \eqref{mytemp3} and \eqref{mytemp4} immediately imply
\begin{align}
\label{mytemp5}
\max(\vct{a}^*\vct{v},0)\le \frac{1}{\left(\frac{3}{2}\frac{1}{(1-\gamma)}-\epsilon\right)}\cdot\max\left(\underset{\vct{u}\in\mathcal{T}'}{\sup}\text{ }\vct{a}^*\vct{u},0\right)
\quad\text{and}\quad\min(\vct{a}^*\vct{v},0)\le\frac{1}{\left(\frac{3}{2}\frac{1}{(1-\gamma)}+\epsilon\right)}\cdot\min\left(\underset{\vct{u}\in\mathcal{T}'}{\sup}\text{ }\vct{a}^*\vct{u},0\right).
\end{align}
By \eqref{mytemp2}, $\sigma(\mathcal{T}')=\underset{\vct{v}\in\mathcal{T}'}{\sup}\twonorm{\vct{v}}\le \frac{3}{2}\frac{1}{(1-\gamma)}+\epsilon$. Thus, using \eqref{mytemp5} we have
\begin{align}
\label{subsetineq}
\omega(\mathcal{T})=&\E[\underset{\vct{v}\in\mathcal{T}}{\sup}\text{ }\vct{a}^*\vct{v}]\nonumber\\
=&\E[\underset{\vct{v}\in\mathcal{T}}{\sup}\text{ }\left(\max(\vct{a}^*\vct{v},0)+\min(\vct{a}^*\vct{v},0)\right)]\nonumber\\
\le&\E\Bigg[\frac{1}{\left(\frac{3}{2}\frac{1}{(1-\gamma)}-\epsilon\right)}\cdot\max\left(\underset{\vct{u}\in\mathcal{T}'}{\sup}\text{ }\vct{a}^*\vct{u},0\right)+\frac{1}{\left(\frac{3}{2}\frac{1}{(1-\gamma)}+\epsilon\right)}\cdot\min\left(\underset{\vct{u}\in\mathcal{T}'}{\sup}\text{ }\vct{a}^*\vct{u},0\right)\Bigg]\nonumber\\
=&\E\Bigg[\frac{2\epsilon}{\left(\frac{9}{4}\frac{1}{(1-\gamma)^2}-\epsilon^2\right)}\cdot\max\left(\underset{\vct{u}\in\mathcal{T}'}{\sup}\text{ }\vct{a}^*\vct{u},0\right)+\frac{1}{\left(\frac{3}{2}\frac{1}{(1-\gamma)}+\epsilon\right)}\cdot\underset{\vct{u}\in\mathcal{T}'}{\sup}\text{ }\vct{a}^*\vct{u}\Bigg]\nonumber\\
=&\frac{2\epsilon}{\left(\frac{9}{4}\frac{1}{(1-\gamma)^2}-\epsilon^2\right)}\cdot\E\Bigg[\max\left(\underset{\vct{u}\in\mathcal{T}'}{\sup}\text{ }\vct{a}^*\vct{u},0\right)\Bigg]+\frac{1}{\left(\frac{3}{2}\frac{1}{(1-\gamma)}+\epsilon\right)}\cdot\omega(\mathcal{T}')\nonumber\\
=&\frac{2\epsilon}{\left(\frac{9}{4}\frac{1}{(1-\gamma)^2}-\epsilon^2\right)}\cdot\int_0^\infty \mathbb{P}\Bigg\{\max\left(\underset{\vct{u}\in\mathcal{T}'}{\sup}\text{ }\vct{a}^*\vct{u},0\right)\ge t\Bigg\}dt+\frac{1}{\left(\frac{3}{2}\frac{1}{(1-\gamma)}+\epsilon\right)}\cdot\omega(\mathcal{T}')\nonumber\\
=&\frac{2\epsilon}{\left(\frac{9}{4}\frac{1}{(1-\gamma)^2}-\epsilon^2\right)}\cdot\int_0^\infty \mathbb{P}\Big\{\underset{\vct{u}\in\mathcal{T}'}{\sup}\text{ }\vct{a}^*\vct{u}\ge t\Big\}dt+\frac{1}{\left(\frac{3}{2}\frac{1}{(1-\gamma)}+\epsilon\right)}\cdot\omega(\mathcal{T}')\nonumber\\
\le&\frac{2\epsilon}{\left(\frac{9}{4}\frac{1}{(1-\gamma)^2}-\epsilon^2\right)}\cdot\int_0^\infty e^{-\frac{\left(t-\omega(\mathcal{T}')\right)^2}{2\sigma^2(\mathcal{T}')}}dt+\frac{1}{\left(\frac{3}{2}\frac{1}{(1-\gamma)}+\epsilon\right)}\cdot\omega(\mathcal{T}')\nonumber\\
=&\frac{2\epsilon}{\left(\frac{9}{4}\frac{1}{(1-\gamma)^2}-\epsilon^2\right)}\cdot\sqrt{\frac{\pi}{2}}\sigma(\mathcal{T}')\left(\text{erf}\left(\frac{\omega(\mathcal{T}')}{\sqrt{2}\sigma(\mathcal{T}')}\right)+1\right)+\frac{1}{\left(\frac{3}{2}\frac{1}{(1-\gamma)}+\epsilon\right)}\cdot\omega(\mathcal{T}')\nonumber\\
\le&\frac{4\epsilon}{\left(\frac{9}{4}\frac{1}{(1-\gamma)^2}-\epsilon^2\right)}\cdot\sqrt{\frac{\pi}{2}}\sigma(\mathcal{T}')+\frac{1}{\left(\frac{3}{2}\frac{1}{(1-\gamma)}+\epsilon\right)}\cdot\omega(\mathcal{T}')\nonumber\\
\le&\frac{1}{\left(\frac{3}{2}\frac{1}{(1-\gamma)}-\epsilon\right)}\cdot\left(\sqrt{8\pi}\epsilon+\omega(\mathcal{T}')\right).
\end{align}
Also using the fact that $\vct{z}\in E(\epsilon)$ we have
\begin{align*}
\omega(\mathcal{T}')=\E[\underset{\vct{u}\in\mathcal{T}'}{\sup}\text{ }\vct{a}^*\vct{u}]=&\E\Big[\underset{\vct{z}\in E(\epsilon)}{\sup} \text{ }\vct{a}^*\left(\frac{(\gamma+\frac{1}{2})}{(1-\gamma)}\vct{x}+\vct{z}\right)\Big]\\
=&\E\Big[\underset{\vct{z}\in E(\epsilon)}{\sup} \text{ }\vct{a}^*\left(\frac{3}{2}\frac{1}{(1-\gamma)}\vct{x}+\vct{z}-\vct{x}\right)\Big]\\
\le&\E\Big[\underset{\vct{z}\in E(\epsilon)}{\sup} \text{ }\vct{a}^*\left(\frac{3}{2}\frac{1}{(1-\gamma)}\vct{x}\right)+\underset{\vct{z}\in E(\epsilon)}{\sup} \text{ }\vct{a}^*\left(\vct{z}-\vct{x}\right)\Big]\\
=&\E\Big[\frac{3}{2}\frac{1}{(1-\gamma)}\vct{a}^*\vct{x}+\underset{\vct{z}\in E(\epsilon)}{\sup} \text{ }\vct{a}^*\left(\vct{z}-\vct{x}\right)\Big]\\
=&\E\Big[\underset{\vct{z}\in E(\epsilon)}{\sup} \text{ }\vct{a}^*\left(\vct{z}-\vct{x}\right)\Big]\\
\le&\epsilon\cdot\E\Big[\underset{\vct{u}\in\mathcal{C}_{\mathcal{R}}(\vct{x})\cap\mathcal{B}^{n}}{\sup} \text{ }\vct{a}^*\vct{u}\Big]\\
=&\epsilon\cdot\omega\left(\mathcal{C}_{\mathcal{R}}(\vct{x})\cap\mathbb{S}^{n-1}\right)
\end{align*}
Now using \eqref{subsetineq} together with the above we have
\begin{align*}
\omega(\mathcal{C}'\cap \mathbb{S}^{n-1})=\omega(\mathcal{T})\le\frac{1}{\left(\frac{3}{2}\frac{1}{(1-\gamma)}+\epsilon\right)}\cdot\left(\sqrt{8\pi}\epsilon+\omega(\mathcal{T}')\right)\le&\frac{\epsilon}{\left(\frac{3}{2}\frac{1}{(1-\gamma)}+\epsilon\right)}\cdot\left(\sqrt{8\pi}+\omega(\mathcal{C}_{\mathcal{R}}(\vct{x})\cap\mathbb{S}^{n-1})\right)\\
\le&\left(\sqrt{8\pi}+\omega\left(\mathcal{C}_{\mathcal{R}}(\vct{x})\cap\mathbb{S}^{n-1}\right)\right).
\end{align*}
Therefore as long as $m\ge \max\left(c\cdot\omega^2\left(\mathcal{C}_{\mathcal{R}}(\vct{x})\cap\mathbb{S}^{n-1}\right),1\right)$ for a fixed numerical constant $c$, applying Lemma \ref{RIP1} with $\vct{u}=\vct{h}$ and $\vct{v}=\vct{x}+\frac{2}{3}(1-\gamma)\vct{h}$ and $\delta=\frac{2}{\pi}\Delta$ with probability at least $1-2e^{-\gamma m}$ we have
\begin{align*}
\frac{1}{m}\sum_{r=1}^m \abs{\vct{h}^*\vct{a}_r\vct{a}_r^*\left(\vct{x}+\frac{2}{3}(1-\gamma)\vct{h}\right)}\ge& \E\Big[ \abs{\vct{h}^*\vct{a}\vct{a}^*\left(\vct{x}+\frac{2}{3}(1-\gamma)\vct{h}\right)}\Big]-\frac{2}{\pi}\Delta\twonorm{\vct{h}}\twonorm{\vct{x}+\frac{2}{3}(1-\gamma)\vct{h}},\\
\ge&\frac{2}{\pi}(1-\Delta)\twonorm{\vct{h}}\twonorm{\vct{x}+\frac{2}{3}(1-\gamma)\vct{h}},
\end{align*}
where in the last inequality we have applied Lemma \ref{expRIP1}. To prove \eqref{eq:reduced_assumption1_alt} it then suffices to show
\begin{align*}
\frac{4}{\pi^2}(1-\Delta)^2\twonorm{\vct{h}}^2\twonorm{\vct{x}+\frac{2}{3}(1-\gamma)\vct{h}}^2\ge \left(\frac{4}{9}(1-\gamma)\left(\frac{1}{\alpha}+\lambda-2(1+\delta)\right)+(1+\delta)\right)\twonorm{\vct{h}}^2+\frac{2}{9}(1+8\gamma)(\vct{h}^*\vct{x})^2.
\end{align*}
Using the fact that $\twonorm{\vct{x}+\frac{2}{3}(1-\gamma)\vct{h}}\ge 1-\frac{2}{3}(1-\gamma)\epsilon$, it suffices to prove
\begin{align*}
\frac{4}{\pi^2}(1-\Delta)^2\left(1-\frac{2}{3}(1-\gamma)\epsilon\right)^2\ge \left(\frac{4}{9}(1-\gamma)\left(\frac{1}{\alpha}+\lambda-2(1+\delta)\right)+(1+\delta)\right)+\frac{2}{9}(1+8\gamma).
\end{align*}
The latter holds as long as
\begin{align}
\label{myineq}
\epsilon\le \frac{3}{2(1-\gamma)}\left(1- \frac{\pi}{2}\frac{1}{(1-\Delta)}\sqrt{\left(\frac{4}{9}(1-\gamma)\left(\frac{1}{\alpha}+\lambda-2(1+\delta)\right)+(1+\delta)\right)+\frac{2}{9}(1+8\gamma)}\right).
\end{align}
Using the values
\begin{align*}
\alpha= 250,\quad\lambda= \frac{1}{250},\quad\gamma=\frac{1}{1000},\quad\delta= \frac{1}{1000},\quad\Delta=\frac{1}{1000},
\end{align*}
the inequality in \eqref{myineq} holds as long as $\epsilon\le\frac{1}{8}$, completing the proof. 

\subsubsection{Proof of the local smoothness condition}
\label{pfLSC}
For any $\vct{z}\in E(\epsilon)$, we want to prove \eqref{LSC}, which
is equivalent to proving that for all $\vct{w}\in\R^n$ obeying
$\twonorm{\vct{w}}=1$, we have
\[
\left|\left(\nabla \mathcal{L}_I(\vct{z})\right)^*\vct{w}\right|^2\leq \beta\left(\lambda\twonorm{\vct{z}-\vct{x}}^2+\frac{\gamma}{m}\sum_{r=1}^m\abs{\vct{a}_r^*(\vct{z}-\vct{x})}^4\right).
\]
Recall that 
\begin{align*}
\nabla \mathcal{L}_I(\vct{z})=\frac{1}{m}\sum_{r=1}^m\left(\abs{\langle\vct{a}_r,\vct{z}\rangle}^2-y_r\right)(\vct{a}_r\vct{a}_r^*)\vct{z}
\end{align*} 
and define 
\begin{align*}
g(\vct{h},\vct{w})=\frac{1}{m}\sum_{r=1}^m \bigg(&2(\vct{h}^*\vct{a}_r)(\vct{w}^*\vct{a}_r)\abs{\vct{a}_r^*\vct{x}}^2+3|\vct{a}_r^*\vct{h}|^2(\vct{w}^*\vct{a}_r)(\vct{a}_r^*\vct{x})+(\vct{a}_r^*\vct{h})^3(\vct{w}^*\vct{a}_r)\bigg).
\end{align*}
Define $\vct{h}:=\vct{z}-\vct{x}$, to establish \eqref{LSC} it suffices to prove that
\begin{align}
\label{temphaa}
\abs{g(\vct{h},\vct{w})}^2\leq \beta\left( \lambda\twonorm{\vct{h}}^2+\frac{\gamma}{m}\sum_{r=1}^m\abs{\vct{a}_r^*\vct{h}}^4\right).
\end{align}
holds for all $\vct{h}$ and $\vct{w}$ satisfying $\twonorm{\vct{h}}\leq \epsilon$ and $\twonorm{\vct{w}}=1$. 
%
Note that since $(a+b+c)^2\le 3(a^2+b^2+c^2)$
\begin{align}
\label{intermediateLSC}
\abs{g(\vct{h},\vct{w})}^2&\le\Bigg|\frac{1}{m}\sum_{r=1}^m \bigg(2|\vct{h}^*\vct{a}_r||\vct{w}^*\vct{a}_r||\vct{a}_r^*\vct{x}|^2+3 |\vct{h}^*\vct{a}_r|^2|\vct{a}_r^*\vct{x}||\vct{w}^*\vct{a}_r|+|\vct{a}_r^*\vct{h}|^3|\vct{w}^*\vct{a}_r|\bigg)\Bigg|^2\nonumber
\\
&\le 3\Bigg|\frac{2}{m}\sum_{r=1}^m |\vct{h}^*\vct{a}_r||\vct{w}^*\vct{a}_r||\vct{a}_r^*\vct{x}|^2\Bigg|^2+3\Bigg|\frac{3}{m}\sum_{r=1}^m|\vct{h}^*\vct{a}_r|^2|\vct{a}_r^*\vct{x}||\vct{w}^*\vct{a}_r|\Bigg|^2+3\Bigg|\frac{1}{m}\sum_{r=1}^m|\vct{a}_r^*\vct{h}|^3|\vct{w}^*\vct{a}_r|\Bigg|^2\nonumber\\
& := 3(I_1 + I_2 + I_3). 
\end{align}
We now bound each of the terms on the right-hand side. For the first term, using Cauchy-Schwarz and applying Lemma \ref{gordontype} and Lemma \ref{GordonExtra} we have
\begin{align}
\label{LSCterm1}
I_1  \le&
\left(\frac{1}{m}\sum_{r=1}^m\abs{\vct{a}_r^*\vct{x}}^2\abs{\vct{a}_r^*\vct{w}}^2\right)\left(\frac{1}{m}\sum_{r=1}^m\abs{\vct{a}_r^*\vct{x}}^2\abs{\vct{a}_r^*\vct{h}}^2\right)\nonumber\\
\le& 6n \left(\frac{1}{m}\sum_{r=1}^m\abs{\vct{a}_r^*\vct{x}}^2\right)\left(\frac{1}{m}\sum_{r=1}^m\abs{\vct{a}_r^*\vct{x}}^2\abs{\vct{a}_r^*\vct{h}}^2\right)\nonumber\\
\le&6n \left(\frac{1}{m}\sum_{r=1}^m\abs{\vct{a}_r^*\vct{x}}^2\right)\left((1+\delta)\twonorm{\vct{h}}^2+2(\vct{x}^*\vct{h})^2\right)\nonumber\\
\le&24(1+\delta)n\twonorm{\vct{h}}^2\nonumber\\
\le&48n\twonorm{\vct{h}}^2.
\end{align}
Similarly, for the second term, we have
\begin{align}
\label{LSCterm2}
I_2 \le& \left(\frac{1}{m}\sum_{r=1}^m\abs{\vct{a}_r^*\vct{h}}^4\right)\left(\frac{1}{m}\sum_{r=1}^m\abs{\vct{a}_r^*\vct{w}}^2\abs{\vct{a}_r^*\vct{x}}^2\right),\nonumber\\
\le& 12n\left(\frac{1}{m}\sum_{r=1}^m\abs{\vct{a}_r^*\vct{h}}^4\right).
\end{align}
Finally, for the third term we use the Cauchy-Schwarz inequality together with  Lemma \ref{gordontype} to derive 
\begin{align}
\label{LSCterm3}
I_3
 \leq\left(\frac{1}{m}\sum_{r=1}^m\abs{\vct{a}_r^*\vct{h}}^3\max_r\twonorm{\vct{a}_r}\right)^2 & \leq
6n\left(\frac{1}{m}\sum_{r=1}^m\abs{\vct{a}_r^*\vct{h}}^3\right)^2,
\nonumber\\
& \le 6n\left(\frac{1}{m}\sum_{r=1}^m\abs{\vct{a}_r^*\vct{h}}^4\right)
\left(\frac{1}{m}\sum_{r=1}^m\abs{\vct{a}_r^*\vct{h}}^2\right),\nonumber\\
& \le 6n(1+\delta)\twonorm{\vct{h}}^2\left(\frac{1}{m}\sum_{r=1}^m\abs{\vct{a}_r^*\vct{h}}^4\right),\nonumber\\
& \le 12n\twonorm{\vct{h}}^2\left(\frac{1}{m}\sum_{r=1}^m\abs{\vct{a}_r^*\vct{h}}^4\right),\nonumber\\
& \le 12n\left(\frac{1}{m}\sum_{r=1}^m\abs{\vct{a}_r^*\vct{h}}^4\right),\nonumber\\
& \le \frac{3}{8}n\left(\frac{1}{m}\sum_{r=1}^m\abs{\vct{a}_r^*\vct{h}}^4\right).
\end{align}
We now plug these inequalities into \eqref{intermediateLSC} and get 
\begin{align}
\label{intermediateLSC2}
\abs{g(\vct{h},\vct{w})}^2&\le 48n\twonorm{\vct{h}}^2+12n\frac{1}{m}\sum_{r=1}^m\abs{\vct{a}_r^*\vct{h}}^4+\frac{3}{8}n\frac{1}{m}\sum_{r=1}^m\abs{\vct{a}_r^*\vct{h}}^4,\nonumber\\
&\le \beta\left(\lambda\twonorm{\vct{h}}^2+\frac{\gamma}{m}\sum_{r=1}^m\abs{\vct{a}_r^*\vct{h}}^4\right), 
\end{align}
which completes the proof of \eqref{temphaa} and, in turn, establishes 
the local smoothness condition in \eqref{LSC}. However, the last line
of \eqref{intermediateLSC2} holds as long as
\begin{align}
\label{betabound}
\beta\ge \max \left( \frac{48}{\lambda},
  \frac{13}{\gamma}\right)n=13000n,
\end{align}
completing the proof.


\subsection{Convergence analysis for amplitude-based Wirtinger Flows}
\label{secAPWF}
In this section we shall prove Theorem \ref{Athm}. Throughout we use the shorthand $\mathcal{C}$ to denote the descent cone of $\mathcal{R}$ at $\vct{x}$, i.e.~$\mathcal{C}=\mathcal{C}_{\mathcal{R}}(\vct{x})$. Note that \eqref{initAPWF2} guarantees that either $\twonorm{\vct{z}_0-\vct{x}}$ or $\twonorm{\vct{z}_0+\vct{x}}$ is small. Throughout the proof without loss of generality we assume  $\twonorm{\vct{z}_0-\vct{x}}$ is the smaller one. To introduce our general convergence analysis we begin again by defining 
\begin{align*}
E(\epsilon)=\big\{\vct{z}\in\R^n:\mathcal{R}(\vct{z})\le \mathcal{R}(\vct{x}),\text{ }\twonorm{\vct{z}-\vct{x}}\le\epsilon\big\}.
\end{align*}
Note that when condition \eqref{initAPWF2} holds the next iterate $\vct{z}_1$ obeys $\vct{z}_1\in E(\epsilon)$ with $\epsilon=2/15$. The reason is that when the sub-level sets of the regularizer is a closed set (not necessarily convex) by Lemma \ref{projcvxthm} equation \eqref{projncvxthm1} we have
\begin{align*}
\twonorm{\vct{z}_1-\vct{x}}=\twonorm{\mathcal{P}_{\mathcal{K}}\left(\vct{z}_0\right)-\vct{x}}\le 2\twonorm{\vct{z}_0-\vct{x}}\le \epsilon.
\end{align*}
To prove Theorem \ref{Athm} note that if we apply the projected Wirtinger Flow update
\begin{align*}
\vct{z}_{\tau+1}=\mathcal{P}_{\mathcal{K}}\left(\vct{z}_\tau-\nabla\mathcal{L}_A(\vct{z}_\tau)\right),
\end{align*}
then the difference between our iterates and the actual solution $\vct{h}_\tau=\vct{z}_\tau-\vct{x}$ is inside the descent set $\mathcal{D}=\mathcal{K}-\{\vct{x}\}$. Thus we have the following chain of  inequalities
\begin{align}
\label{interpfthm12}
\tn{\vct{h}_{\tau+1}}=\tn{\des_{\tau+1}-\vct{x}}&=\tn{\Pc_{\mathcal{K}}\left(\des_{\tau}-\nabla\mathcal{L}_A(\vct{z}_\tau)\right)-\x},\nonumber\\
&\overset{(a)}{=}\tn{\Pc_{\mathcal{K}-\{\vct{x}\}}\left(\des_{\tau}-\vct{x}-\nabla\mathcal{L}_A(\vct{z}_\tau)\right)},\nonumber\\
&=\tn{\Pc_{\mathcal{D}}\left(\vct{z}_{\tau}-\vct{x}-\nabla\mathcal{L}_A(\vct{z}_\tau)\right)},\nonumber\\
&\overset{(b)}{\leq} 2\tn{\Pc_\Cc\left(\des_{\tau}-\vct{x}-\nabla\mathcal{L}_A(\vct{z}_\tau)\right)},\nonumber\\
&=2\tn{\Pc_\Cc\left(\h_{\tau}-\nabla\mathcal{L}_A(\vct{z}_\tau)\right)},\nonumber\\
&\leq 2\cdot\sup_{\vct{u}\in\Cc\cap\Bc^{n}} \vct{u}^*\left(\h_{\tau}-\nabla\mathcal{L}_A(\vct{z}_\tau)\right).
\end{align}
In the inequalities above (a) follows from Lemma \ref{simplem} and (b) follows from Lemma \ref{prop compare}. To complete the convergence analysis it is then sufficient to prove
\begin{align}
\label{mainintemp}
\sup_{\vct{u}\in\Cc\cap\Bc^{n}} \vct{u}^*\left(\h_{\tau}-\nabla\mathcal{L}_A(\vct{z}_\tau)\right)\le\frac{1}{3}\twonorm{\h_{\tau}}=\frac{1}{3}\twonorm{\vct{z}_\tau-\vct{x}}.
\end{align}
We will instead prove that the following stronger result holds for all $\vct{u}\in\Cc\cap\Bc^{n}$ and $\vct{z}\in E(\epsilon)$
\begin{align}
\label{mainineq}
\vct{u}^*\left(\vct{z}-\vct{x}-\nabla\mathcal{L}_A(\vct{z})\right)\le\frac{1}{3}\twonorm{\vct{z}-\vct{x}}.
\end{align}
The equation \eqref{mainineq} above implies \eqref{mainintemp} which when combined with \eqref{interpfthm12} proves the convergence result of the Theorem (specifically equation \eqref{ratemin}).

The rest of this section is dedicated to proving \eqref{mainineq}. To this aim note that for $\vct{u}\in\mathcal{C}$ we have
\begin{align*}
\langle\nabla \mathcal{L}_A(\vct{z}) ,\vct{u}\rangle=&\frac{1}{m}\sum_{r=1}^m\left(\vct{a}_r^*\vct{z}-\abs{\vct{a}_r^*\vct{x}}\sgn{\vct{a}_r^*\vct{z}}\right)(\vct{a}_r^*\vct{u}),\\
=&\frac{1}{m}\sum_{r=1}^m\left(\vct{a}_r^*\vct{z}\right)\left(\vct{a}_r^*\vct{u}\right)-\frac{1}{m}\sum_{r=1}^m \abs{\vct{a}_r^*\vct{x}}(\vct{a}_r^*\vct{u})\sgn{\vct{a}_r^*\vct{z}},\\
=&\frac{1}{m}\sum_{r=1}^m\left(\vct{a}_r^*(\vct{z}-\vct{x})\right)\left(\vct{a}_r^*\vct{u}\right)+\frac{1}{m}\sum_{r=1}^m(\vct{a}_r^*\vct{x})(\vct{a}_r^*\vct{u})-\frac{1}{m}\sum_{r=1}^m \abs{\vct{a}_r^*\vct{x}}(\vct{a}_r^*\vct{u})\sgn{\vct{a}_r^*\vct{z}},\\
=&\frac{1}{m}\sum_{r=1}^m\left(\vct{a}_r^*(\vct{z}-\vct{x})\right)\left(\vct{a}_r^*\vct{u}\right)-\frac{1}{m}\sum_{r=1}^m\abs{\vct{a}_r^*\vct{x}}\left(\sgn{\vct{a}_r^*\vct{z}}-\sgn{\vct{a}_r^*\vct{x}}\right)(\vct{a}_r^*\vct{u}).
\end{align*}
Thus, 
\begin{align}
\label{maindev}
\vct{u}^*\left(\vct{z}-\vct{x}-\nabla \mathcal{L}_A(\vct{z})\right)=&\vct{u}^*\left(\mtx{I}-\frac{1}{m}\mtx{A}^*\mtx{A}\right)(\vct{z}-\vct{x})+\frac{1}{m}\sum_{r=1}^m\abs{\vct{a}_r^*\vct{x}}\left(\sgn{\vct{a}_r^*\vct{z}}-\sgn{\vct{a}_r^*\vct{x}}\right)(\vct{a}_r^*\vct{u}).
\end{align}
We now proceed by bounding each of the two summands in \eqref{maindev}. For the first term by Lemma \ref{gordontypenonsym}, as long as 
\begin{align*}
m\ge \max\left(80\frac{m_0}{\delta^2},\frac{2}{\delta}-1\right),
\end{align*}
then for all $\vct{u}\in\Cc\cap\Bc^{n}$ and $\vct{z}\in E(\epsilon)$
\begin{align}
\label{myinterAWF}
\vct{u}^*\left(\mtx{I}-\frac{1}{m}\mtx{A}^*\mtx{A}\right)(\vct{z}-\vct{x})\le\delta \twonorm{\vct{z}-\vct{x}},
\end{align}
holds with probability at least $1-6e^{-\frac{\delta^2}{1440}m}$. To bound the second term we make use of the following lemma whose proof is deferred to Section \ref{secondterm}.
\begin{lemma}\label{implemAWF} Let $\mathcal{C}\subset\R^n$ be the descent cone of the regularization function $\mathcal{R}$ at $\vct{x}$ i.e.~$\mathcal{C}=\mathcal{C}_{\mathcal{R}}(\vct{x})$. Set $\epsilon=1/8$ and define the set
\begin{align*}
E(\epsilon)=\big\{\vct{z}\in\R^n\text{ }|\text{ }\vct{z}\in\mathcal{K}\emph{ and }\twonorm{\vct{z}-\vct{x}}\le \epsilon\big\}.
\end{align*}
Assume the measurement vectors $\{\vct{a}_r\}_{r=1}^m$ are i.i.d.~Gaussian random vectors distributed as $\mathcal{N}(\vct{0},\mtx{I})$. Furthermore, assume $\vct{x}\in\R^n$ is a fixed vector independent of the measurement vectors, obeying $\twonorm{\vct{x}}=1$. Furthermore, as long as
\begin{align*}
m\ge c\cdot\omega^2\left(\mathcal{C}\cap\mathbb{S}^{n-1}\right),
\end{align*}
holds for a fixed numerical constant $c$, then for all $\vct{u}\in\Cc\cap\Bc^{n}$ and $\vct{z}\in E(\epsilon)$,
\begin{align*}
\frac{1}{m}\sum_{r=1}^m\abs{\vct{a}_r^*\vct{x}}\left(\sgn{\vct{a}_r^*\vct{z}}-\sgn{\vct{a}_r^*\vct{x}}\right)(\vct{a}_r^*\vct{u})\le\frac{32}{100}\twonorm{\vct{z}-\vct{x}},
\end{align*}
holds with probability at least $1-3e^{-\gamma m}$ with $\gamma$ a fixed numerical constant.
\end{lemma}
Plugging \eqref{myinterAWF} with $\delta=1/100$ and Lemma \ref{myinterAWF} into \eqref{maindev} yields \eqref{mainineq} concluding the proof of Theorem \ref{Athm}. The rest of this section is dedicated to the proof of Lemma \ref{implemAWF}.
\subsubsection{Bounding the second term in \eqref{maindev} (Proof of Lemma \ref{implemAWF})}
\label{secondterm}
For the second term note that a simple bound together with the Cauchy-Schwarz inequality gives
\begin{align}
\label{newdev1}
\frac{1}{m}\sum_{r=1}^m\abs{\vct{a}_r^*\vct{x}}\left(\sgn{\vct{a}_r^*\vct{z}}-\sgn{\vct{a}_r^*\vct{x}}\right)(\vct{a}_r^*\vct{u})\le&\frac{1}{m}\sum_{r=1}^m\abs{\vct{a}_r^*\vct{x}}\abs{\vct{a}_r^*\vct{u}}\left(1-\sgn{\vct{a}_r^*\vct{x}}\sgn{\vct{a}_r^*\vct{z}}\right),\nonumber\\
\le&\sqrt{\frac{1}{m}\sum_{r=1}^m\abs{\vct{a}_r^*\vct{x}}^2\left(1-\sgn{\vct{a}_r^*\vct{x}}\sgn{\vct{a}_r^*\vct{z}}\right)^2}
\sqrt{\frac{1}{m}\sum_{r=1}^m\abs{\vct{a}_r^*\vct{u}}^2},\nonumber\\
=&\sqrt{\frac{1}{m}\sum_{r=1}^m\abs{\vct{a}_r^*\vct{x}}^2\left(2-2\sgn{\vct{a}_r^*\vct{x}}\sgn{\vct{a}_r^*\vct{z}}\right)}
\sqrt{\frac{1}{m}\sum_{r=1}^m\abs{\vct{a}_r^*\vct{u}}^2}.
\end{align}
Note that by Lemma \ref{gordontype} as long as $m\ge \max\left(80\frac{\omega^2\left(\mathcal{C}\cap\mathbb{S}^{n-1}\right)}{\delta^2},\frac{2}{\delta}-1\right)$, then
\begin{align*}
\frac{1}{m}\sum_{r=1}^m\abs{\vct{a}_r^*\vct{u}}^2\le 1+\delta,
\end{align*}
holds for all $\vct{u}\in\mathcal{C}\cap\mathbb{S}^{n-1}$ with probability at least $1-2e^{-\gamma\delta^2 m}$. Combining the latter inequality with \eqref{newdev1} we conclude that with high probability
\begin{align}
\label{newdev2}
\frac{1}{m}\sum_{r=1}^m\abs{\vct{a}_r^*\vct{x}}\left(\sgn{\vct{a}_r^*\vct{z}}-\sgn{\vct{a}_r^*\vct{x}}\right)(\vct{a}_r^*\vct{u})\le\sqrt{2(1+\delta)}\sqrt{\frac{1}{m}\sum_{r=1}^m\abs{\vct{a}_r^*\vct{x}}^2\left(1-\sgn{\vct{a}_r^*\vct{x}}\sgn{\vct{a}_r^*\vct{z}}\right)}.
\end{align}

We now turn our attention to bounding the right-hand side of \eqref{newdev2}. To this aim first note that as long as $\twonorm{\vct{z}-\vct{x}}\le \epsilon$ we have $\vct{z}^T\vct{x}\ge 1-\epsilon$. Thus, we have the following chain of inequalities 
\begin{align}
\label{firstbnd}
\frac{1}{m}\sum_{r=1}^m\abs{\vct{a}_r^*\vct{x}}^2\left(1-\sgn{\vct{a}_r^*\vct{z}}\sgn{\vct{a}_r^*\vct{x}}\right)=&\frac{2}{m}\sum_{r=1}^m\abs{\vct{a}_r^*\vct{x}}^2\mathbb{1}_{\{(\vct{a}_r^*\vct{z})(\vct{a}_r^*\vct{x})\le 0\}}\nonumber\\
=&\frac{2}{m}\sum_{r=1}^m\abs{\vct{a}_r^*\vct{x}}^2\mathbb{1}_{\{\left((\vct{z}^*\vct{x})(\vct{a}_r^*\vct{x})+\vct{a}_r^*\vct{h}_{\perp}\right)(\vct{a}_r^*\vct{x})\le 0\}}\nonumber\\
=&\frac{2}{m}\sum_{r=1}^m\abs{\vct{a}_r^*\vct{x}}^2\mathbb{1}_{\{(\vct{z}^*\vct{x})\abs{\vct{a}_r^*\vct{x}}^2\le-(\vct{a}_r^*\vct{h}_{\perp})(\vct{a}_r^*\vct{x})\}}\nonumber\\
\le&\frac{2}{m}\sum_{r=1}^m\abs{\vct{a}_r^*\vct{x}}^2\mathbb{1}_{\big\{\abs{\vct{a}_r^*\vct{x}}\le\abs{\frac{\vct{a}_r^*\vct{h}_{\perp}}{(\vct{z}^*\vct{x})}}\big\}}\nonumber\\
\le&\frac{2}{(\vct{z}^*\vct{x})^2}\frac{1}{m}\sum_{r=1}^m\abs{\vct{a}_r^*\vct{h}_{\perp}}^2\mathbb{1}_{\big\{\abs{\vct{a}_r^*\vct{x}}\le\abs{\frac{\vct{a}_r^*\vct{h}_{\perp}}{(\vct{z}^*\vct{x})}}\big\}}\nonumber\\
\le&\frac{2}{(1-\epsilon)^2}\frac{1}{m}\sum_{r=1}^m\abs{\vct{a}_r^*\vct{h}_{\perp}}^2\mathbb{1}_{\big\{(1-\epsilon)\abs{\vct{a}_r^*\vct{x}}\le\abs{\vct{a}_r^*\vct{h}_{\perp}}\big\}}.
\end{align}
Now define i.i.d.~Gaussian random variables $w_r$ distributed as $\mathcal{N}(0,1)$ and independent from $\vct{a}_r$ and note that
\begin{align}
\label{hperp}
\frac{1}{m}\sum_{r=1}^m\abs{\vct{a}_r^*\vct{h}_{\perp}}^2\mathbb{1}_{\big\{(1-\epsilon)\abs{\vct{a}_r^*\vct{x}}\le\abs{\vct{a}_r^*\vct{h}_{\perp}}\big\}},
\end{align}
has the same distribution as
\begin{align}
\label{hperpc}
\frac{1}{m}\sum_{r=1}^m\abs{\vct{a}_r^*\vct{h}_{\perp}}^2\mathbb{1}_{\big\{(1-\epsilon)\abs{\vct{w}_r}\le\abs{\vct{a}_r^*\vct{h}_{\perp}}\big\}}.
\end{align}
Hence to bound \eqref{hperp} (and hence \eqref{newdev2} via \eqref{firstbnd}) with high probability, it suffices to bound \eqref{hperpc} with high probability. 
To proceed we define the mapping $\mathcal{S}:\R \times \R_{+}\rightarrow \R_{+}$ as
\begin{align}
\label{myS}
\mathcal{S}(h;\beta)=
\left\{
	\begin{array}{ll}
		0 &\quad\quad\quad\text{ } \abs{h} < \beta(1-\Delta) \\
		\frac{1}{\Delta}\left(\abs{h}-\beta(1-\Delta)\right)  & \beta(1-\Delta)\le \abs{h} \le \beta \\
		\abs{h}  & \quad\quad\quad\text{ } \abs{h} > \beta 
	\end{array}
\right.
\end{align}
Note that 
\begin{align*}
\abs{\vct{a}_r^*\vct{h}_{\perp}}\mathbb{1}_{\big\{(1-\epsilon)\abs{\vct{w}_r}\le\abs{\vct{a}_r^*\vct{h}_{\perp}}\big\}}\le\abs{\vct{a}_r^*\vct{h}}\mathbb{1}_{\big\{(1-\epsilon)\abs{\vct{w}_r}\le\abs{\vct{a}_r^*\vct{h}}\big\}}\le\mathcal{S}\left(\vct{a}_r^*\vct{h};(1-\epsilon)\abs{\vct{w}_r}\right)
\end{align*}
Hence to bound \eqref{hperp} (and hence \eqref{newdev2} via \eqref{firstbnd}) with high probability, it suffices to bound 
\begin{align}
\label{sbnd}
\frac{1}{m}\sum_{r=1}^m \mathcal{S}^2\left(\vct{a}_r^*\vct{h};(1-\epsilon)\abs{\vct{w}_r}\right),
\end{align}
with high probability. To bound this term we need to utilize two lemmas whose proofs are deferred to Sections \ref{proofoflem} and \ref{proofkeylem}. 
\begin{lemma}
\label{lemsqrtconc}
Let $\vct{z},\vct{\beta}\in\R^m$. Define the function
\begin{align}
\label{}
f(\vct{z};\vct{\beta})=\sqrt{\sum_{r=1}^m\mathcal{S}^2\left(\vct{z}_r;\vct{\beta}_r\right)},
\end{align}
with $\mathcal{S}$ as defined in \eqref{myS}. Then this function obeys two key properties
\begin{itemize}
\item \textbf{Lipschitz:} i.e.~the function $f$ obeys
\begin{align}
\label{Lip}
\abs{f(\vct{z};\vct{\beta})-f(\vct{y};\vct{\beta})}\le \frac{1}{\Delta}\twonorm{\vct{z}-\vct{y}}.
\end{align}
\item \textbf{Radial convexity:} i.e.~the function $f$ obeys 
\begin{align}
\label{radCVX}
f(\alpha\vct{z};\vct{\beta})\le \alpha f(\vct{z};\vct{\beta}),
\end{align}
for all $\vct{z}$ and $0\le\alpha\le 1$.
\end{itemize}
\end{lemma} 
\begin{lemma} \label{keylemma} Let $f:\R^n\rightarrow \R_{+}$ be a function obeying \eqref{Lip} with Lipschitz constant $L$.  Furthermore, let $\mtx{A}$ be an $m\times n$ matrix with i.i.d.~$\mathcal{N}(0,1)$ entries. Then for any subset $\mathcal{T}\subset\epsilon\mathbb{S}^{n-1}$ of the sphere of radius $\epsilon$, we have
\begin{align}
\label{keylemeq}
\underset{\vct{z}\in\mathcal{T}}{\sup}\text{ }f(\mtx{A}\vct{z})-\E[f(\mtx{A}\vct{z})]\le CL\left(\omega(\mathcal{T})+t\epsilon\right),
\end{align}
with probability at least $1-2e^{-\frac{t^2}{4}}$. Here, $\omega(\mathcal{T})$ is the Gaussian width of the set $\mathcal{T}$ per Definition \ref{Gausswidth}. 
\end{lemma}
With these lemmas in hand we are now ready to bound \eqref{sbnd}. To this aim note that
\begin{align*}
\frac{1}{m}\sum_{r=1}^m \mathcal{S}^2\left(\vct{a}_r^*\vct{h};(1-\epsilon)\abs{\vct{w}_r}\right)=\frac{1}{m}f^2\left(\mtx{A}\vct{h};(1-\epsilon)\abs{\vct{w}}\right).
\end{align*}
Thus using the radial convexity property \eqref{radCVX} from Lemma \ref{lemsqrtconc}, we conclude that for all $\vct{h}\in\mathcal{C}\cap\epsilon\mathcal{B}^n$ we have
\begin{align}
\label{intr}
\frac{1}{m}\sum_{r=1}^m \frac{\mathcal{S}^2\left(\vct{a}_r^*\vct{h};(1-\epsilon)\abs{\vct{w}_r}\right)}{\twonorm{\vct{h}}^2}=&\frac{1}{m}\frac{f^2\left(\mtx{A}\vct{h};(1-\epsilon)\abs{\vct{w}}\right)}{\twonorm{\vct{h}}^2},\nonumber\\
=&\frac{1}{m}\frac{f^2\left(\frac{\twonorm{\vct{h}}}{\epsilon}\mtx{A}\frac{\epsilon\vct{h}}{\twonorm{\vct{h}}};(1-\epsilon)\abs{\vct{w}}\right)}{\twonorm{\vct{h}}^2},\nonumber\\
\le&\frac{1}{m}\frac{\frac{\twonorm{\vct{h}}^2}{\epsilon^2}}{\twonorm{\vct{h}}^2}f^2\left(\mtx{A}\frac{\epsilon\vct{h}}{\twonorm{\vct{h}}};(1-\epsilon)\abs{\vct{w}}\right),\nonumber\\
=&\frac{1}{m\epsilon^2}f^2\left(\mtx{A}\frac{\epsilon\vct{h}}{\twonorm{\vct{h}}};(1-\epsilon)\abs{\vct{w}}\right),\nonumber\\
=&\frac{1}{m\epsilon^2}\Bigg(f\left(\mtx{A}\frac{\epsilon\vct{h}}{\twonorm{\vct{h}}};(1-\epsilon)\abs{\vct{w}}\right)-\E\bigg[f\left(\mtx{A}\frac{\epsilon\vct{h}}{\twonorm{\vct{h}}};(1-\epsilon)\abs{\vct{w}}\right)\bigg]\nonumber\\
&\quad\quad\quad+\E\bigg[f\left(\mtx{A}\frac{\epsilon\vct{h}}{\twonorm{\vct{h}}};(1-\epsilon)\abs{\vct{w}}\right)\bigg]\Bigg)^2,\nonumber\\
\le&\frac{1}{m\epsilon^2}\Bigg(\underset{\vct{u}\in\mathcal{C}\cap \epsilon\mathbb{S}^{n-1}}{\sup}\left(f\left(\mtx{A}\vct{u};(1-\epsilon)\abs{\vct{w}}\right)-\E\bigg[f\left(\mtx{A}\vct{u};(1-\epsilon)\abs{\vct{w}}\right)\bigg]\right)\nonumber\\
&\quad\quad\quad+\E\bigg[f\left(\mtx{A}\frac{\epsilon\vct{h}}{\twonorm{\vct{h}}};(1-\epsilon)\abs{\vct{w}}\right)\bigg]\Bigg)^2.
\end{align}
Now note that by Lemma \ref{lemsqrtconc}, $f\left(\mtx{A}\vct{u};(1-\epsilon)\abs{\vct{w}}\right)$ is $L=1/\Delta$ Lipschitz. Thus, applying Lemma \ref{keylemma} with $\mathcal{T}=\mathcal{C}\cap\epsilon\mathbb{S}^{n-1}$ and $t=\frac{\Delta}{C}\frac{\delta}{2}\sqrt{m}$ we have
\begin{align*}
\underset{\vct{u}\in\mathcal{C}\cap \epsilon\mathbb{S}^{n-1}}{\sup}\text{ }f(\mtx{A}\vct{u};(1-\epsilon)\abs{\vct{w}})-\E[f(\mtx{A}\vct{u};(1-\epsilon)\abs{\vct{w}})]\le \frac{C}{\Delta}\epsilon\cdot\omega(\mathcal{C}\cap\mathbb{S}^{n-1})+\frac{\delta}{2}\epsilon\sqrt{m},
\end{align*}
holds with probability at least $1-e^{-\gamma\delta^2m}$ with $\gamma$ a fixed numerical constant. Combining the latter inequality with \eqref{intr} we conclude that 
\begin{align}
\label{intr2}
\frac{1}{m}\sum_{r=1}^m \frac{\mathcal{S}^2\left(\vct{a}_r^*\vct{h};(1-\epsilon)\abs{\vct{w}_r}\right)}{\twonorm{\vct{h}}^2}\le\left(\frac{C}{\Delta}\frac{\omega(\mathcal{C}\cap\mathbb{S}^{n-1})}{\sqrt{m}}+\frac{\delta}{2}+\frac{1}{\sqrt{m}\epsilon}\E\bigg[f\left(\mtx{A}\frac{\epsilon\vct{h}}{\twonorm{\vct{h}}};(1-\epsilon)\abs{\vct{w}}\right)\bigg]\right)^2,
\end{align}
holds with high probability. Now combining \eqref{intr2} with Jenson's inequality we conclude that
\begin{align}
\label{intr3}
\frac{1}{m}\sum_{r=1}^m \frac{\mathcal{S}^2\left(\vct{a}_r^*\vct{h};(1-\epsilon)\abs{\vct{w}_r}\right)}{\twonorm{\vct{h}}^2}\le&\left(\frac{C}{\Delta}\frac{\omega(\mathcal{C}\cap\mathbb{S}^{n-1})}{\sqrt{m}}+\frac{\delta}{2}+\frac{1}{\sqrt{m}\epsilon}\sqrt{\E\bigg[f^2\left(\mtx{A}\frac{\epsilon\vct{h}}{\twonorm{\vct{h}}};(1-\epsilon)\abs{\vct{w}}\right)\bigg]}\right)^2,\nonumber\\
=&\left(\frac{C}{\Delta}\frac{\omega(\mathcal{C}\cap\mathbb{S}^{n-1})}{\sqrt{m}}+\frac{\delta}{2}+\frac{1}{\sqrt{m}\epsilon}\sqrt{\E\bigg[\sum_{r=1}^m\mathcal{S}^2\left(\vct{a}_r^*\frac{\epsilon\vct{h}}{\twonorm{\vct{h}}};(1-\epsilon)\abs{\vct{w}_r}\right)\bigg]}\right)^2,
\end{align}
holds with high probability. Now note that using $s_r=\frac{(1-\epsilon)\abs{\vct{w}_r}}{\twonorm{\vct{v}}}$ and $\Delta=0.1$ we have
\begin{align*}
\E\big[\mathcal{S}^2(\vct{a}_r^*\vct{v};(1-\epsilon)\abs{\vct{w}_r})\big]=&\frac{1}{\Delta^2}\sqrt{\frac{2}{\pi}}\int_{(1-\epsilon)(1-\Delta)\abs{\vct{w}_r}}^{(1-\epsilon)\abs{\vct{w}_r}}(x-(1-\epsilon)(1-\Delta)\abs{\vct{w}_r})^2e^{-\frac{x^2}{2\twonorm{\vct{v}}^2}}dx\\
&+\sqrt{\frac{2}{\pi}}\int_{(1-\epsilon)\abs{\vct{w}_r}}^{\infty}x^2e^{-\frac{x^2}{2\twonorm{\vct{v}}^2}}dx\\
=&\frac{2}{\Delta^2}\twonorm{\vct{v}}\left((1-\epsilon)^2(1-\Delta)^2\abs{\vct{w}_r}^2+\twonorm{\vct{v}}^2\right)\left(Q\left(\frac{(1-\epsilon)(1-\Delta)\abs{\vct{w}_r}}{\twonorm{\vct{v}}}\right)-Q\left(\frac{(1-\epsilon)\abs{\vct{w}_r}}{\twonorm{\vct{v}}}\right)\right)\\
&+\sqrt{\frac{2}{\pi}}\frac{\abs{(1-\epsilon)\vct{w}_r}}{\Delta^2}\twonorm{\vct{v}}^2\left((1-2\Delta)e^{-\frac{(1-\epsilon)^2\abs{\vct{w}_r}^2}{2\twonorm{\vct{v}}^2}}-(1-\Delta)e^{-\frac{(1-\epsilon)^2(1-\Delta)^2\abs{\vct{w}_r}^2}{2\twonorm{\vct{v}}^2}}\right)\\
&+2\twonorm{\vct{v}}^3Q\left(\frac{(1-\epsilon)\abs{\vct{w}_r}}{\twonorm{\vct{v}}}\right)+\sqrt{\frac{2}{\pi}}(1-\epsilon)\abs{\vct{w}_r}\twonorm{\vct{v}}^2e^{-\frac{(1-\epsilon)^2\abs{\vct{w}_r}^2}{2\twonorm{\vct{v}}^2}}\\
=&\twonorm{\vct{v}}^3\left(\frac{2}{\Delta^2}\left((1-\Delta)^2s_r^2+1\right)\left(Q\left((1-\Delta)s_r\right)-Q(s_r)\right)\right)\\
&+\twonorm{\vct{v}}^3\left(\sqrt{\frac{2}{\pi}}\frac{1}{\Delta^2}s_r\left((1-2\Delta)e^{-\frac{s_r^2}{2}}-(1-\Delta)e^{-\frac{(1-\Delta)^2s_r^2}{2}}\right)+2Q(s_r)+\sqrt{\frac{2}{\pi}}s_re^{-\frac{s_r^2}{2}}\right)\\
\le&\frac{21}{20}\twonorm{\vct{v}}^3.
\end{align*}
Using the latter inequality with $\vct{v}=\frac{\epsilon \vct{h}}{\twonorm{\vct{h}}}$ we conclude that
\begin{align*}
\E\bigg[\sum_{r=1}^m\mathcal{S}^2\left(\vct{a}_r^*\frac{\epsilon\vct{h}}{\twonorm{\vct{h}}};(1-\epsilon)\abs{\vct{w}_r}\right)\bigg]= m\cdot\E\bigg[\mathcal{S}^2\left(\vct{a}_r^*\frac{\epsilon\vct{h}}{\twonorm{\vct{h}}};(1-\epsilon)\abs{\vct{w}_r}\right)\bigg]\le \frac{21}{20}m\epsilon^3.
\end{align*}
Combining the latter with \eqref{intr3} we conclude that as long as $m\ge \frac{4C^2}{\Delta^2}\frac{\omega^2(\mathcal{C}\cap\mathbb{S}^{n-1})}{\delta^2}=400C^2\frac{\omega^2(\mathcal{C}\cap\mathbb{S}^{n-1})}{\delta^2}$,
\begin{align*}
\frac{1}{m}\sum_{r=1}^m \mathcal{S}^2\left(\vct{a}_r^*\vct{h};(1-\epsilon)\abs{\vct{w}_r}\right)\le \left(\frac{C}{\Delta}\frac{\omega(\mathcal{C}\cap\mathbb{S}^{n-1})}{\sqrt{m}}+\frac{\delta}{2}+\sqrt{\frac{21}{20}\epsilon}\right)^2\twonorm{\vct{h}}^2\le \left(\delta+\sqrt{\frac{21}{20}}\epsilon\right)^2\twonorm{\vct{h}}^2,
\end{align*}
holds with high probability. The latter inequality provides a bound on \eqref{sbnd} and in turn bounds \eqref{hperpc} and \eqref{hperp}. Plugging this into \eqref{firstbnd} and \eqref{newdev2} we conclude that
\begin{align*}
\frac{1}{m}\sum_{r=1}^m\abs{\vct{a}_r^*\vct{x}}\left(\sgn{\vct{a}_r^*\vct{z}}-\sgn{\vct{a}_r^*\vct{x}}\right)(\vct{a}_r^*\vct{u})\le&\sqrt{2(1+\delta)}\sqrt{\frac{1}{m}\sum_{r=1}^m\abs{\vct{a}_r^*\vct{x}}^2\left(1-\sgn{\vct{a}_r^*\vct{x}}\sgn{\vct{a}_r^*\vct{z}}\right)},\\
\le& 2\frac{\sqrt{(1+\delta)}}{(1-\epsilon)}\sqrt{\frac{1}{m}\sum_{r=1}^m\abs{\vct{a}_r^*\vct{h}_{\perp}}^2\mathbb{1}_{\big\{(1-\epsilon)\abs{\vct{a}_r^*\vct{x}}\le\abs{\vct{a}_r^*\vct{h}_{\perp}}\big\}}},\\
\le&2\frac{\sqrt{(1+\delta)}}{(1-\epsilon)}\left(\delta+\sqrt{\frac{21}{20}}\epsilon\right)\twonorm{\vct{h}},\\
\le&\frac{32}{100}\twonorm{\vct{h}}.
\end{align*}
holds with high probability. Where in the last inequality we have used $\epsilon=\frac{2}{15}$ and $\delta=1/1000$. This completes the proof of the lemma and the bound on the second term.

\subsubsection{Proof of Lemma (Lemma \eqref{lemsqrtconc})}
\label{proofoflem}

We begin by proving two useful properties for the function $f$. 

\noindent\textbf{Proof of Lipschitzness:} Note that
\begin{align*}
\frac{d}{dh}\mathcal{S}(h;\beta)=\left\{
	\begin{array}{ll}
		0 &\quad\quad\quad\text{ } \abs{h} < \beta(1-\Delta) \\
		\frac{1}{\Delta}\sgn{h}  & \beta(1-\Delta)\le \abs{h} \le \beta \\
		\sgn{h}  & \quad\quad\quad\text{ } \abs{h} > \beta 
	\end{array}
\right.
\end{align*}
Thus $\mathcal{S}(h;\beta_r)$ is Lipschitz with Lipschitz constant $L=1/\Delta$ which implies 
\begin{align*}
\abs{f(\vct{y};\vct{\beta})-f(\vct{x};\vct{\beta})}\le\sqrt{\sum_{r=1}^m\abs{\mathcal{S}(\vct{y}_r;\vct{\beta}_r)-\mathcal{S}(\vct{x}_r;\vct{\beta}_r)}^2}\le\sqrt{L^2\left(\sum_{r=1}^m\abs{\vct{y}_r-\vct{x}_r}^2\right)}=L\twonorm{\vct{y}-\vct{x}}.
\end{align*}

\noindent\textbf{Proof of Radial convexity:}
Note that for all $h\in\R$, $\beta\in\R_{+}$ and $0\le\alpha\le 1$ we have
\begin{align*}
\mathcal{S}(\alpha h;\beta)\le\alpha \mathcal{S}(h;\beta).
\end{align*}
Thus for all $\vct{x}\in\R^n$, $\vct{\beta}\in\R_{+}^m$ and $0\le\alpha\le 1$ we have
\begin{align*}
f(\alpha\vct{x};\vct{\beta})=\sqrt{\sum_{r=1}^m\mathcal{S}^2(\alpha \vct{x}_r;\vct{\beta}_r)}\le\sqrt{\alpha^2\sum_{r=1}^m\mathcal{S}^2(\vct{x}_r;\vct{\beta}_r)}=\alpha f(\vct{x};\vct{\beta}),
\end{align*}
completing the proof of radial convexity.

\subsubsection{Proof of Key Lemma (Lemma \eqref{keylemma})}
\label{proofkeylem}
This proof is heavily inspired by \cite[Theorem 11.1.6]{HDP} with some steps directly borrowed. However, a few modifications are applied when necessary. To prove this lemma we make use of a powerful result due to Talagrand (see also \cite[Theorem 3.2]{dirksen2015tail}). We state the version of this result stated in \cite[Theorem 4.1]{liaw2016simple}.
\begin{lemma}\label{talag}[Talagrand's Tail bound] Let $(X_{\vct{x}})_{\vct{x}\in\mathcal{T}}$ be a random process on a subset $\mathcal{T}\subset \R^n$. Assume that for all $\vct{x},\vct{y}\in\mathcal{T}$, we have
\begin{align*}
\|X_{\vct{x}}-X_{\vct{y}}\|_{\psi_2}\le K\twonorm{\vct{x}-\vct{y}}.
\end{align*}
The for every $u\ge 0$ we have
\begin{align*}
\underset{\vct{x}\in\mathcal{T}}{\sup}\text{ }X_{\vct{x}}\le CK\left(\omega(\mathcal{T})+u\cdot\text{diam}(\mathcal{T})\right),
\end{align*}
with probability at least $1-2e^{-u^2}$.
\end{lemma}
To apply this lemma define the random process
\begin{align*}
X_{\vct{x}}:=f(\mtx{A}\vct{x})-\E[f(\mtx{A}\vct{x})].
\end{align*}
We shall prove that $X_{\vct{x}}$ has sub-gaussian increments with respect to the Euclidean norm, namely
\begin{align}
\label{sginc}
\|X_{\vct{x}}-X_{\vct{y}}\|_{\psi_2}\le c L\twonorm{\vct{x}-\vct{y}}\quad\text{for all}\quad\vct{x},\vct{y}\in\mathcal{T}.
\end{align}
The latter inequality together with Lemma \ref{talag} immediately implies \eqref{keylemeq} concluding the proof. All that remains is to show that the sub-Gaussian property \eqref{sginc} indeed holds. To this aim, first note that without loss of generality we may assume that $L=1$. We also note that since $\mathcal{T}$ is a subset of the unit sphere of radius $\epsilon$ ($\mathcal{T}\subset \epsilon \mathbb{S}^{n-1}$) $\twonorm{\vct{x}}=\twonorm{\vct{y}}=\epsilon$ for all $\vct{x},\vct{y}\in\mathcal{T}$. When $\twonorm{\vct{x}}=\twonorm{\vct{y}}=\epsilon$, $\mtx{A}\vct{x}$ and $\mtx{A}\vct{y}$ have the same marginal distribution (not necessarily independent), that is $\mtx{A}\vct{x}\sim \epsilon \vct{g}$ and $\mtx{A}\vct{y}\sim \epsilon \vct{g'}$ with $\vct{g},\vct{g'}\sim\mathcal{N}(\vct{0},\mtx{I})$. Thus
\begin{align*}
\E[f(\mtx{A}\vct{x})]=\E[f(\epsilon \vct{g})]=\E[f(\epsilon \vct{g}')]=\E[f(\mtx{A}\vct{y})].
\end{align*}
Hence, proving the inequality in \eqref{sginc} reduces to proving
\begin{align}
\label{reduc}
\|f(\mtx{A}\vct{x})-f(\mtx{A}\vct{y})\|_{\psi_2}\le C\twonorm{\vct{x}-\vct{y}}.
\end{align}
Following \cite{HDP} we proceed by creating independence by considering the vectors
\begin{align*}
\vct{u}:=\frac{\vct{x}+\vct{y}}{2},\quad\vct{v}:=\frac{\vct{x}-\vct{y}}{2}.
\end{align*}
Then $\vct{x}=\vct{u}+\vct{v}$ and $\vct{y}=\vct{u}-\vct{v}$, and thus 
\begin{align*}
\mtx{A}\vct{x}=\mtx{A}\vct{u}+\mtx{A}\vct{v}\quad\text{and}\quad \mtx{A}\vct{y}=\mtx{A}\vct{u}-\mtx{A}\vct{v}.
\end{align*}
Note that the vectors $\vct{u}$ and $\vct{v}$ are orthogonal which implies that the Gaussian random vectors $\mtx{A}\vct{u}$ and $\mtx{A}\vct{v}$ are independent. Now note that conditioned on $\mtx{A}\vct{u}$, $\mtx{A}\vct{v}$ is a Gaussian random vector that can be expressed as 
\begin{align*}
\mtx{A}\vct{u}+\mtx{A}\vct{v}=\mtx{A}\vct{u}+\twonorm{\vct{v}}\vct{g},\quad\text{where}\quad\vct{g}\sim\mathcal{N}(\vct{0},\mtx{I}_m). 
\end{align*}
Note that the Lipschitz property implies that for any $\tilde{\vct{g}}$ and $\vct{g}$
\begin{align*}
\abs{f(\mtx{A}\vct{u}+\twonorm{\vct{v}}\tilde{\vct{g}})-f(\mtx{A}\vct{u}+\twonorm{\vct{v}}\vct{g})}\le \twonorm{\vct{v}}\twonorm{\tilde{\vct{g}}-\vct{g}},
\end{align*}
so that $f(\mtx{A}\vct{u}+\twonorm{\vct{v}}\vct{g})$ is a Lipschitz function of $\vct{g}$ (Lipschitz constant equal to $\twonorm{\vct{v}}$). Thus by concentration of Lipschitz functions of Gaussians we conclude that conditioned on $\mtx{A}\vct{u}$ we have
\begin{align}
\label{fA1}
\|f(\mtx{A}\vct{u}+\mtx{A}\vct{v})-\E[f(\mtx{A}\vct{u}+\mtx{A}\vct{v})]\|_{\psi_2}\le C\twonorm{\vct{v}}.
\end{align}
Now note that since the random vector $\mtx{A}\vct{u}-\mtx{A}\vct{v}$ has the same distribution as $\mtx{A}\vct{u}+\mtx{A}\vct{v}$, conditioned on $\mtx{A}\vct{u}$ it satisfies the same bound. That is,
\begin{align}
\label{fA2}
\|f(\mtx{A}\vct{u}-\mtx{A}\vct{v})-\E[f(\mtx{A}\vct{u}-\mtx{A}\vct{v})]\|_{\psi_2}\le C\twonorm{\vct{v}}.
\end{align}
Combining \eqref{fA1} and \eqref{fA2} via the triangular inequality we conclude that conditioned on $\mtx{A}\vct{u}$
\begin{align*}
\|f(\mtx{A}\vct{u}+\mtx{A}\vct{v})-f(\mtx{A}\vct{u}-\mtx{A}\vct{v})\|_{\psi_2}\le 2C\twonorm{\vct{v}},
\end{align*}
While this bound is for the conditional distribution, it also holds for the original distribution as its is true for any fixed $\mtx{A}\vct{u}$. So that 
\begin{align*}
\|f(\mtx{A}\vct{u}+\mtx{A}\vct{v})-f(\mtx{A}\vct{u}-\mtx{A}\vct{v})\|_{\psi_2}\le 2C\twonorm{\vct{v}}.
\end{align*}
holds without any conditioning. Rewriting $\vct{u}$ and $\vct{v}$ in terms of $\vct{x}$ and $\vct{y}$ we conclude that
\begin{align*}
\|f(\mtx{A}\vct{x})-f(\mtx{A}\vct{y})\|_{\psi_2}\le C\twonorm{\vct{x}-\vct{y}}.
\end{align*}
This completes the proof of \eqref{sginc}.
\subsection*{Acknowledgements}
M.S. would like to thank Martin Wainwright for helpful discussions and feedback during the early stages of this work while he was a postdoc at UC Berkeley. He would also like to thank Emmanuel Candes, Samet Oymak, Xiaodong Li, Laura Waller, Lei Tian and John Wright for many stimulating discussions related to this paper.

\bibliography{Bibfiles}
\bibliographystyle{plain}

\appendix 
\section{Proof of preliminary lemmas}\label{proofofsimp}
\subsection{Proof of Lemma \ref{gordontype}}
\label{pfgordontype}
This result is rather straightforward consequence of Gordon's escape through the mesh lemma stated below.
\begin{theorem}[Gordon's escape through the mesh]\label{GDlemma} Let $\delta\in(0,1)$, $\mathcal{C}\subset\R^n$ be a subset of the unit sphere ($\mathcal{C}\subset\mathbb{S}^{n-1}$) and let $\mtx{A}\in\R^{m\times n}$ be a matrix with i.i.d $\mathcal{N}(0,1)$ entries. Furthermore, set $b_m=\sqrt{2}\frac{\Gamma\left(\frac{m+1}{2}\right)}{\Gamma\left(\frac{m}{2}\right)}$. Then, 
\begin{align}
\label{Gisometry}
\abs{\frac{\twonorm{\mtx{A}\vct{x}}}{b_m}-\twonorm{\vct{x}}}\le \delta\twonorm{\vct{x}},
\end{align}
holds for all $\x\in\mathcal{C}$ with probability at least $1-2e^{-\frac{\eta^2}{2}}$ as long as
\begin{align}
\label{nummeas}
m\ge\frac{\left(\omega(\mathcal{T})+\eta\right)^2}{\delta^2}.
\end{align}
\end{theorem}
From Gordon's lemma above together with the fact that $\delta\le 1$ we conclude that
\begin{align}
\label{GordonQuad}
-2\delta\twonorm{\vct{x}}^2\le \frac{\twonorm{\mtx{A}\vct{x}}^2}{b_m^2}-\twonorm{\vct{x}}^2\le3\delta\twonorm{\vct{x}}^2
\end{align}
Using \eqref{GordonQuad} together with the fact that $b_m^2\le m$ we conclude that
\begin{align}
\label{oneside}
\frac{\twonorm{\mtx{A}\vct{x}}^2}{m}-\twonorm{\vct{x}}^2\le \frac{\twonorm{\mtx{A}\vct{x}}^2}{b_m^2}-\twonorm{\vct{x}}^2 \le3\delta\twonorm{\vct{x}}^2.
 \end{align}
Using \eqref{GordonQuad} together with the fact that $b_m^2\ge m-\frac{1}{2}$ and $m\ge \frac{1}{2\delta}-1$ we conclude that
\begin{align}
\label{otherside}
-2\delta\twonorm{\vct{x}}^2\le\frac{\twonorm{\mtx{A}\vct{x}}^2}{m-\frac{1}{2}}-\twonorm{\vct{x}}^2\Rightarrow \frac{\twonorm{\mtx{A}\vct{x}}^2}{m} -\twonorm{\vct{x}}^2\ge\left(-2\delta+\frac{\delta}{m}-\frac{1}{2m}\right)\twonorm{\vct{x}}^2\ge -3\delta \twonorm{\vct{x}}^2.
\end{align}
Combining \eqref{oneside} and \eqref{otherside} and replacing $\delta$ with $\delta/3$ we conclude that as long as
\begin{align*}
m\ge \max\left(9\frac{\left(\omega(\mathcal{C}\cap\mathbb{S}^{n-1})+\eta\right)^2}{\delta^2},\frac{3}{2\delta}-1\right),
\end{align*}
then
\begin{align*}
\abs{\frac{1}{m}\twonorm{\mtx{A}\vct{x}}^2-\twonorm{\vct{x}}^2}\le\delta\twonorm{\vct{x}}^2,
\end{align*}
holds with probability at least $1-2e^{-\frac{\eta^2}{2}}$. Using $\eta=\frac{\delta}{6\sqrt{5}}\sqrt{m}$ and the inequality $(a+b)^2\le 2a^2+2b^2$ as long as
\begin{align*}
m\ge \max\left(20\frac{\omega^2(\mathcal{C}\cap\mathbb{S}^{n-1})}{\delta^2},\frac{3}{2\delta}-1\right),
\end{align*}
then
\begin{align*}
\abs{\frac{1}{m}\twonorm{\mtx{A}\vct{x}}^2-\twonorm{\vct{x}}^2}\le\delta\twonorm{\vct{x}}^2,
\end{align*}
holds with probability at least $1-2e^{-\frac{\delta^2}{360}m}$.

\subsection{Proof of Lemma \ref{gordontypenonsym}}\label{appendix2}
Note that
\begin{align}
\label{mylemma42}
\frac{1}{m}\sum_{r=1}^m \left(\vct{a}_r^*\frac{\vct{u}}{\twonorm{\vct{u}}}\right)\left(\vct{a}_r^*\frac{\vct{h}}{\twonorm{\vct{h}}}\right)-\frac{(\vct{u}^*\vct{h})}{\twonorm{\vct{u}}\twonorm{\vct{h}}}=&2\left(\frac{1}{m}\sum_{r=1}^m\left(\vct{a}_r^*\left(\frac{\vct{u}}{2\twonorm{\vct{u}}}+\frac{\vct{h}}{2\twonorm{\vct{h}}}\right)\right)^2-\twonorm{\frac{\vct{u}}{2\twonorm{\vct{u}}}+\frac{\vct{h}}{2\twonorm{\vct{h}}}}^2\right)\nonumber\\
&-\left(\frac{1}{m}\sum_{r=1}^m\left(\vct{a}_r^*\left(\frac{\vct{u}}{\twonorm{\vct{u}}}\right)\right)^2-1\right)-\left(\frac{1}{m}\sum_{r=1}^m\left(\vct{a}_r^*\left(\frac{\vct{h}}{\twonorm{\vct{h}}}\right)\right)^2-1\right)
\end{align}
Now note that 
\begin{align*}
\frac{\vct{u}}{2\twonorm{\vct{u}}}+\frac{\vct{h}}{2\twonorm{\vct{h}}}\in\frac{1}{2}\mathcal{C}\cap\mathbb{S}^{n-1}+\frac{1}{2}\mathcal{C}\cap\mathbb{S}^{n-1}
\end{align*}
Thus by Lemma \ref{gordontype} as long as
\begin{align*}
m\ge& \max\left(80\frac{\omega^2\left(\frac{1}{2}\mathcal{C}\cap\mathbb{S}^{n-1}+\frac{1}{2}\mathcal{C}\cap\mathbb{S}^{n-1}\right)}{\delta^2},\frac{2}{\delta}-1\right),\\
\ge &\max\left(80\frac{\omega^2\left(\mathcal{C}\cap\mathbb{S}^{n-1}\right)}{\delta^2},\frac{2}{\delta}-1\right)
\end{align*}
then for all $\vct{u},\vct{h}\in\mathcal{C}$
\begin{align*}
\abs{\frac{1}{m}\sum_{r=1}^m\left(\vct{a}_r^*\left(\frac{\vct{u}}{2\twonorm{\vct{u}}}+\frac{\vct{h}}{2\twonorm{\vct{h}}}\right)\right)^2-\twonorm{\frac{\vct{u}}{2\twonorm{\vct{u}}}+\frac{\vct{h}}{2\twonorm{\vct{h}}}}^2}\le \frac{\delta}{4}
\end{align*}
holds with probability at least $1-2e^{-\frac{\delta^2}{1440}m}$. Similarly, as long as
\begin{align*}
m\ge \max\left(80\frac{\omega^2\left(\mathcal{C}\cap\mathbb{S}^{n-1}\right)}{\delta^2},\frac{2}{\delta}-1\right)
\end{align*}
then
\begin{align*}
\abs{\frac{1}{m}\sum_{r=1}^m\left(\vct{a}_r^*\left(\frac{\vct{u}}{\twonorm{\vct{u}}}\right)\right)^2-1}\le& \frac{\delta}{4}\\
\abs{\frac{1}{m}\sum_{r=1}^m\left(\vct{a}_r^*\left(\frac{\vct{h}}{\twonorm{\vct{h}}}\right)\right)^2-1}\le& \frac{\delta}{4}
\end{align*}
holds with probability at least $1-4e^{-\frac{\delta^2}{1440}m}$. Now uses these bounds via the triangular inequality in \eqref{mylemma42} we conclude that 
\begin{align*}
\abs{\frac{1}{m}\sum_{r=1}^m (\vct{a}_r^*\vct{u})(\vct{a}_r^*\vct{h})-\vct{u}^*\vct{h}}\le \delta\twonorm{\vct{u}}\twonorm{\vct{h}}
\end{align*}
holds with probability at least $1-6e^{-\frac{\delta^2}{1440}m}$.
\subsection{Proof of Lemma \ref{GTtypelem}}
\label{pfgordontype2}
Our proof is related to the proof of Gordon's celebrated escape through the mesh. We will first show the bound $\twonorm{\mtx{D}\mtx{A}\vct{u}}\le b_m(\vct{d})\twonorm{\vct{u}}+\infnorm{\vct{d}}\omega(\mathcal{T})+\eta$. For $\vct{u}\in\mathcal{T}$ and $\vct{v}\in\mathbb{S}^{m-1}=\{\vct{v}\in\R^m;\text{ }\twonorm{\vct{v}}=1\}$, we define the two Gaussian processes
\begin{align*}
X_{\vct{u},\vct{v}}=\vct{v}^*\mtx{D}\mtx{A}\vct{u}\quad\text{and}\quad Y_{\vct{u},\vct{v}}=\twonorm{\vct{u}}\vct{a}^*\mtx{D}\vct{v}+\twonorm{\mtx{D}\vct{v}}\vct{g}^*\vct{u}.
\end{align*}
Here $\vct{a}\in\R^m$ is distributed as $\mathcal{N}(\vct{0},\mtx{I}_m)$ and $\vct{g}\in\R^n$ is distributed as $\mathcal{N}(\vct{0},\mtx{I}_n)$. It follows that for all $\vct{u},\vct{u}'\in\mathcal{T}$ and $\vct{v},\vct{v}'\in\mathbb{S}^{m-1}$, we have
\begin{align*}
\mathbb{E}\abs{Y_{\vct{u},\vct{v}}-Y_{\vct{u}',\vct{v}'}}^2-\mathbb{E}\abs{X_{\vct{u},\vct{v}}-X_{\vct{u}',\vct{v}'}}^2=&\twonorm{\twonorm{\vct{u}}\mtx{D}\vct{v}-\twonorm{\vct{u}'}\mtx{D}\vct{v}'}^2+\twonorm{\twonorm{\mtx{D}\vct{v}}\vct{u}-\twonorm{\mtx{D}\vct{v}'}\vct{u}'}^2\\
&-\fronorm{\vct{u}(\mtx{D}\vct{v})^*-\vct{u}'(\mtx{D}\vct{v}')^*}^2\\
=&\twonorm{\vct{u}}^2\twonorm{\mtx{D}\vct{v}}^2+\twonorm{\vct{u}'}^2\twonorm{\mtx{D}\vct{v}'}^2+2\langle \vct{u},\vct{u}'\rangle\langle \mtx{D}\vct{v},\mtx{D}\vct{v}'\rangle\\
&-2\twonorm{\vct{u}}\twonorm{\vct{u}'}\langle\mtx{D}\vct{v},\mtx{D}\vct{v}'\rangle-2\twonorm{\mtx{D}\vct{v}}\twonorm{\mtx{D}\vct{v}'}\langle\vct{u},\vct{u}'\rangle\\
=&\twonorm{\twonorm{\vct{u}}\twonorm{\mtx{D}\vct{v}}-\twonorm{\vct{u}'}\twonorm{\mtx{D}\vct{v}'}}^2\\
&+2(\twonorm{\mtx{D}\vct{v}}\twonorm{\mtx{D}\vct{v}'}-\langle\mtx{D}\vct{v},\mtx{D}\vct{v}'\rangle)(\twonorm{\vct{u}}\twonorm{\vct{u}'}-\langle\vct{u},\vct{u}'\rangle)\\
\ge& 0,
\end{align*}
with equality if $\vct{u}=\vct{u}'$ and $\mtx{D}\vct{v}=\mtx{D}\vct{v}'$.

We note that by standard concentration of measure for Gaussian random variables
\begin{align*}
\mathbb{P}\Big\{\twonorm{\mtx{D}\vct{a}}\ge b_m(\vct{d})+\eta\Big\}\le e^{-\frac{\eta^2}{2\infnorm{\vct{d}}^2}}.
\end{align*}
We also have
\begin{align*}
\Big\{\vct{a}: \quad\twonorm{\vct{u}}\twonorm{\mtx{D}\vct{a}}\ge \twonorm{\vct{u}}b_m(\vct{d})+\eta\Big\}\subset&\Big\{\vct{a}: \quad\twonorm{\vct{u}}\twonorm{\mtx{D}\vct{a}}\ge \twonorm{\vct{u}}b_m(\vct{d})+\twonorm{\vct{u}}\frac{\eta}{\sigma(\mathcal{T})}\Big\}\\
=&\Big\{\vct{a}: \quad\twonorm{\mtx{D}\vct{a}}\ge b_m(\vct{d})+\frac{\eta}{\sigma(\mathcal{T})}\Big\}\\
=&\Big\{\vct{a}: \quad\twonorm{\mtx{D}\vct{a}}\ge b_m(\vct{d})+\frac{\eta}{\sigma(\mathcal{T})}\Big\}.
\end{align*}
Thus, 
\begin{align}
\mathbb{P}\Bigg\{\underset{\vct{u}\in\mathcal{T}}{\bigcup}\big\{\vct{a}:\quad\twonorm{\vct{u}}\twonorm{\mtx{D}\vct{a}}&>\twonorm{\vct{u}} b_m(\vct{d})+\eta_1\big\}\Bigg\}\nonumber\\
&\le\mathbb{P}\Big\{\vct{a}:\text{ }\twonorm{\mtx{D}\vct{a}}\ge b_m(\vct{d})+\frac{\eta_1}{\sigma(\mathcal{T})}\Big\}\le e^{-\frac{\eta_1^2}{2\infnorm{\vct{d}}^2\sigma^2(\mathcal{T})}},
\end{align}
which immediately implies
\begin{align}
\label{myfirstineqGT}
\mathbb{P}\Big\{\underset{\vct{u}\in\mathcal{T}}{\max}\text{ }\twonorm{\vct{u}}\left(\twonorm{\mtx{D}\vct{a}}-b_m(\vct{d})\right)>\frac{\eta}{2}\Big\}\le e^{-\frac{\eta^2}{8\infnorm{\vct{d}}^2\sigma^2(\mathcal{T})}}.
\end{align}
Also 
\begin{align}
\label{mysecondineqGT}
\mathbb{P}\Big\{\text{ }\infnorm{\vct{d}}\underset{\vct{u}\in\mathcal{T}}{\max} \text{ }\left(\vct{g}^*\vct{u}\right)>\infnorm{\vct{d}}\omega(\mathcal{T})+\frac{\eta}{2}\Big\}=\mathbb{P}\Big\{\text{ }\infnorm{\vct{d}}\underset{\vct{u}\in\mathcal{T}}{\max} \text{ }\left(\vct{g}^*\vct{u}\right)>\infnorm{\vct{d}}\E\big[\underset{\vct{u}\in\mathcal{T}}{\max} \text{ }\left(\vct{g}^*\vct{u}\right)\big]+\frac{\eta}{2}\Big\}\le e^{-\frac{\eta^2}{8\infnorm{\vct{d}}^2\sigma^2(\mathcal{T})}}.
\end{align}
%
Note that if
\begin{align*}
\underset{\vct{u}\in\mathcal{T}, \vct{v}\in\mathbb{S}^{n-1}}{\max}\text{ }\left(\twonorm{\vct{u}}\left(\vct{a}^*\mtx{D}\vct{v}-b_m(\vct{d})\right)+\twonorm{\mtx{D}\vct{v}}\left((\vct{g}^*\vct{u})-\omega(\mathcal{T})\right)\right)>\eta,
\end{align*}
then either
\begin{align*}
\underset{\vct{u}\in\mathcal{T}, \vct{v}\in\mathbb{S}^{n-1}}{\max}\text{ }\twonorm{\vct{u}}\left(\vct{a}^*\mtx{D}\vct{v}-b_m(\vct{d})\right)=\underset{\vct{u}\in\mathcal{T}}{\max}\text{ }\twonorm{\vct{u}}\left(\twonorm{\mtx{D}\vct{a}}-b_m(\vct{d})\right)>\frac{\eta}{2},
\end{align*}
or
\begin{align*}
\underset{\vct{u}\in\mathcal{T}, \vct{v}\in\mathbb{S}^{n-1}}{\max}\text{ }\twonorm{\mtx{D}\vct{v}}\left((\vct{g}^*\vct{u})-\omega(\mathcal{T})\right)=\infnorm{\vct{d}}\cdot\underset{\vct{u}\in\mathcal{T}}{\max} \text{ }\left(\vct{g}^*\vct{u}\right)-\infnorm{\vct{d}}\omega(\mathcal{T})>\frac{\eta}{2}.
\end{align*}
This implies that
\begin{align*}
\big\{\vct{a},\vct{g}:\text{ }\underset{\vct{u}\in\mathcal{T}}{\max}\text{ }\left(\twonorm{\vct{u}}\twonorm{\mtx{D}\vct{a}}+\infnorm{\vct{d}}\vct{g}^*\vct{u}-b_m(\vct{d})\twonorm{\vct{u}}\right)>\infnorm{\vct{d}}\omega(\mathcal{T})+\eta\big\}
\end{align*}
is a subset of 
\begin{align*}
\big\{\vct{a},\vct{g}:\text{ }\underset{\vct{u}\in\mathcal{T}}{\max}\text{ }\twonorm{\vct{u}}\left(\twonorm{\mtx{D}\vct{a}}-b_m(\vct{d})\right)>\frac{\eta}{2}\big\}\bigcup\big\{\vct{a},\vct{g}:\text{ }\infnorm{\vct{d}}\cdot\underset{\vct{u}\in\mathcal{T}}{\max} \text{ }\left(\vct{g}^*\vct{u}\right)>\infnorm{\vct{d}}\omega(\mathcal{T})+\frac{\eta}{2}\big\}.
\end{align*}
Using the latter together with \eqref{myfirstineqGT} and \eqref{mysecondineqGT} and using the independence of $\vct{a}$ and $\vct{g}$ we have
\begin{align}
\label{mythirdineqGT}
\mathbb{P}\Big\{\underset{\vct{u}\in\mathcal{T}, \vct{v}\in\mathbb{S}^{n-1}}{\max}\text{ }&\left(\twonorm{\vct{u}}\left(\vct{a}^*\mtx{D}\vct{v}-b_m(\vct{d})\right)+\twonorm{\mtx{D}\vct{v}}\left((\vct{g}^*\vct{u})-\omega(\mathcal{T})\right)\right)>\eta\Big\}\nonumber\\
\le& \mathbb{P}\Big\{\underset{\vct{u}\in\mathcal{T}}{\max}\text{ }\twonorm{\vct{u}}\left(\twonorm{\mtx{D}\vct{a}}-b_m(\vct{d})\right)>\frac{\eta}{2}\Big\}+\mathbb{P}\Big\{\text{ }\infnorm{\vct{d}}\cdot\underset{\vct{u}\in\mathcal{T}}{\max} \text{ }\left(\vct{g}^*\vct{u}\right)>\infnorm{\vct{d}}\omega(\mathcal{T})+\frac{\eta}{2}\Big\}\nonumber\\
\le& 2e^{-\frac{\eta^2}{8\infnorm{\vct{d}}^2\sigma^2(\mathcal{T})}}.
\end{align}
The latter inequality is equivalent to
\begin{align*}
\mathbb{P}\Big\{\underset{\vct{u}\in\mathcal{T},\vct{v}\in\mathbb{S}^{n-1}}{\bigcup}[Y_{\vct{u},\vct{v}}>b_m(\vct{d})\twonorm{\vct{u}}+\infnorm{\vct{d}}\omega(\mathcal{T})+\eta]\Big\}\le 2e^{-\frac{\eta^2}{8\infnorm{\vct{d}}^2\sigma^2(\mathcal{T})}}.
\end{align*}
Now using Slepian's second inequality with $\eta_{\vct{u},\vct{v}}=b_m(\vct{d})\twonorm{\vct{u}}+\eta$ we have
\begin{align*}
\mathbb{P}\Big\{\underset{\vct{u}\in\mathcal{T},\vct{v}\in\mathbb{S}^{n-1}}{\bigcup}[X_{\vct{u},\vct{v}}>b_m(\vct{d})\twonorm{\vct{u}}&+\infnorm{\vct{d}}\omega(\mathcal{T})+\eta]\Big\}\nonumber\\
&\le\mathbb{P}\Big\{\underset{\vct{u}\in\mathcal{T},\vct{v}\in\mathbb{S}^{n-1}}{\bigcup}[Y_{\vct{u},\vct{v}}>b_m(\vct{d})\twonorm{\vct{u}}+\infnorm{\vct{d}}\omega(\mathcal{T})+\eta]\Big\}\le 2e^{-\frac{\eta^2}{8\infnorm{\vct{d}}^2\sigma^2(\mathcal{T})}}.
\end{align*}
Noting that
\begin{align*}
\mathbb{P}\Big\{\underset{\vct{u}\in\mathcal{T},\text{ }\vct{v}\in\mathbb{S}^{n-1}}{\max}\text{ }\mtx{X}_{\vct{u},\vct{v}}>\twonorm{\vct{u}}b_m(\vct{d})+\infnorm{\vct{d}}\omega(\mathcal{T})+\eta\Big\}=\mathbb{P}\Big\{\underset{\vct{u}\in\mathcal{T},\text{ }\vct{v}\in\mathbb{S}^{n-1}}{\bigcup}[X_{\vct{u},\vct{v}}>b_m(\vct{d})\twonorm{\vct{u}}+\infnorm{\vct{d}}\omega(\mathcal{T})+\eta]\Big\},
\end{align*}
concludes the proof. 

Next, we show that $\twonorm{\mtx{D}\mtx{A}\vct{u}}\ge b_m(\vct{d})\twonorm{\vct{u}}-\infnorm{\vct{d}}\omega(\mathcal{T})-\eta$. To accomplish this, we make use \cite[Lemma 5.1]{OymLAS}. 
\begin{corollary}\cite[Lemma 5.1]{OymLAS} Let $\mtx{A}\in\R^{m\times n},\vct{g}\in\R^n,\vct{h}\in\R^m$ be independent vectors with independent $\mathcal{N}(0,1)$ entries. Then, for any $c\in\R$
\begin{align*}
\mathbb{P}(\min_{\vct{u}\in \mathcal{T}}\max_{\vct{v}\in\mathcal{T}'}\vct{v}^*\mtx{A}\vct{u}-\phi(\vct{u},\vct{v})\geq c)\geq 2\mathbb{P}(\min_{\vct{u}\in\mathcal{T}}\max_{\vct{v}\in\mathcal{T}'}\vct{v}^*\vct{h}\twonorm{\vct{u}}-\vct{u}^*\vct{g}\twonorm{\vct{v}}-\phi(\vct{u},\vct{v})\geq c)-1.
\end{align*}
\end{corollary}
As long as $\mtx{D}$ is invertible, the previous lemma implies that
\begin{align}
\label{myeqd}
\mathbb{P}(\min_{\vct{u}\in \mathcal{T}}\max_{\vct{v}\in\mathbb{S}^{m-1}}\text{ }\vct{v}^*\mtx{D}\mtx{A}\vct{u}-b_m(\vct{d})&\twonorm{\vct{u}}\geq c)\nonumber\\
&\geq 2\mathbb{P}(\min_{\vct{u}\in\mathcal{T}}\max_{\vct{v}\in\mathbb{S}^{m-1}}\text{ }\vct{v}^*\mtx{D}\vct{h}\twonorm{\vct{u}}-\vct{u}^*\vct{g}\twonorm{\mtx{D}\vct{v}}-b_m(\vct{d})\twonorm{\vct{u}}\geq c)-1.
\end{align}
Assume $\vct{u}^*\vct{g}$ is non-negative, using $\vct{v}=\frac{\mtx{D}\vct{h}}{\twonorm{\mtx{D}\vct{h}}}$ and the definition of max we have
\begin{align*}
\max_{\vct{v}\in\mathbb{S}^{m-1}}\text{ }\vct{v}^*\mtx{D}\vct{h}\twonorm{\vct{u}}-\vct{u}^*\vct{g}\twonorm{\mtx{D}\vct{v}}\ge& \twonorm{\mtx{D}\vct{h}}\twonorm{\vct{u}}-(\vct{u}^*\vct{g})\twonorm{\mtx{D}\frac{\mtx{D}\vct{h}}{\twonorm{\mtx{D}\vct{h}}}},\nonumber\\
\ge&\twonorm{\mtx{D}\vct{h}}\twonorm{\vct{u}}-(\vct{u}^*\vct{g})\infnorm{\vct{d}}.
\end{align*}
Utilizing the latter inequality in \eqref{myeqd} we arrive at 
\begin{align}
\label{myeqd2}
\mathbb{P}(\min_{\vct{u}\in \mathcal{T}}\max_{\vct{v}\in\mathbb{S}^{m-1}}\text{ }\vct{v}^*\mtx{D}\mtx{A}\vct{u}-b_m(\vct{d})&\twonorm{\vct{u}}\geq c)\nonumber\\
&\geq 2\mathbb{P}(\min_{\vct{u}\in\mathcal{T}}\text{ }\twonorm{\mtx{D}\vct{h}}\twonorm{\vct{u}}-(\vct{u}^*\vct{g})\infnorm{\vct{d}}-b_m(\vct{d})\twonorm{\vct{u}}\geq c)-1.
\end{align}
Now using $c=-\omega(\mathcal{T})\infnorm{\vct{d}}-\eta$ in the above inequality we arrive at
\begin{align}
\label{myeqq}
\mathbb{P}\bigg(\min_{\vct{u}\in \mathcal{T}}\max_{\vct{v}\in\mathbb{S}^{m-1}}\text{ }\vct{v}^*\mtx{D}\mtx{A}\vct{u}-b_m(\vct{d})\twonorm{\vct{u}}&\le -\infnorm{\vct{d}}\omega(\mathcal{T})-\eta\bigg)\nonumber\\
&\le 2\mathbb{P}\left(\min_{\vct{u}\in \mathcal{T}} (\twonorm{\mtx{D}\vct{h}}-b_m(\vct{d}))\twonorm{\vct{u}}-\infnorm{\vct{d}}\left(\vct{u}^*\vct{g}-\omega(\mathcal{T})\right)\le-\eta\right).
\end{align}
Note that if $\min_{\vct{u}\in \mathcal{T}} (\twonorm{\mtx{D}\vct{h}}-b_m(\vct{d}))\twonorm{\vct{u}}-\infnorm{\vct{d}}\left(\vct{u}^*\vct{g}-\omega(\mathcal{T})\right)\le-\eta$, then either
\begin{align*}
\min_{\vct{u}\in \mathcal{T}}\text{ }(\twonorm{\mtx{D}\vct{h}}-b_m(\vct{d}))\twonorm{\vct{u}}\le -\frac{\eta}{2}
\end{align*}
or
\begin{align*}
\max_{\vct{u}\in \mathcal{T}}\text{ }\infnorm{\vct{d}}\left(\vct{u}^*\vct{g}-\omega(\mathcal{T})\right)\ge \frac{\eta}{2}.
\end{align*}
Thus applying the union bound in \eqref{myeqq} with some simple algebraic manipulations yields
\begin{align*}
\mathbb{P}\bigg(\min_{\vct{u}\in \mathcal{T}}\max_{\vct{v}\in\mathbb{S}^{m-1}}\text{ }\vct{v}^*\mtx{D}\mtx{A}\vct{u}-b_m&(\vct{d})\twonorm{\vct{u}}\le -\infnorm{\vct{d}}\omega(\mathcal{T})-\eta\bigg),\\
&\le 2\mathbb{P}\left(\twonorm{\mtx{D}\vct{h}}-b_m(\vct{d})\le -\frac{\eta}{2\sigma(\mathcal{T})}\right)+2\mathbb{P}\left(\max_{\vct{u}\in \mathcal{T}}\text{ }\vct{u}^*\vct{g}-\omega(\mathcal{T})\ge \frac{\eta}{2\infnorm{\vct{d}}}\right),\\
&\le4e^{-\frac{\eta^2}{4\infnorm{\vct{d}}^2\sigma^2(\mathcal{T})}}.
\end{align*}
\subsection{Proof of Lemma \ref{GordonExtra}}
\label{pfgordontype3}
Without loss of generality we assume $\twonorm{\vct{x}}$ and $\vct{h}\in\mathcal{C}\cap\mathbb{S}^{n-1}$. To prove this lemma note that
\begin{align}
\label{main44}
\abs{\frac{1}{m}\sum_{r=1}^m\left(\vct{a}_r^*\vct{h}\right)^2\left(\vct{a}_r^*\vct{x}\right)^2-\left(\twonorm{\vct{h}}^2+2(\vct{h}^*\vct{x})^2\right)}\le&(\vct{h}^*\vct{x})^2\abs{\frac{1}{m}\sum_{r=1}^m\left(\vct{a}_r^*\vct{x}\right)^4-3}\nonumber\\
&+\abs{\frac{1}{m}\sum_{r=1}^m(\vct{a}_r^*\vct{x})^2\left(\vct{a}_r^*(\vct{h}-(\vct{h}^*\vct{x})\vct{x})\right)^2-\twonorm{\vct{h}-(\vct{h}^*\vct{x})\vct{x}}^2}\nonumber\\
&+2\abs{\vct{h}^*\vct{x}}\abs{\frac{1}{m}\sum_{r=1}^m(\vct{a}_r^*\vct{x})^3\left(\vct{a}_r^*(\vct{h}-(\vct{h}^*\vct{x})\vct{x})\right)}.
\end{align}
We now proceed by bounding each of the terms separately. To bound the first term of \eqref{main44} we use the Chebyshev's inequality to conclude that for $m\ge\frac{1536}{\delta^2}$
\begin{align*}
\mathbb{P}\left(\abs{\frac{1}{m}\sum_{r=1}^m\left(\vct{a}_r^*\vct{x}\right)^4-3}\ge \frac{\delta}{4}\right)\le \frac{1536}{\delta^2m^2}\le \frac{1}{m},
\end{align*}
holds with probability at least $1-\frac{1}{m}$.

To bound the second term in \eqref{main44} note that for all $r$ the random variables $\vct{a}_r^*\vct{x}$ and $\vct{a}_r^*\left(\vct{h}-(\vct{h}^*\vct{x})\vct{x}\right)$ are independent from each other and therefore we can utilize Corollary \ref{coroG}. to conclude that as long as
\begin{align}
\label{sampE}
\sum_{r=1}^m (\vct{a}_r^*\vct{x})^2\ge \max\left(80\left(\max_{r}\abs{\vct{a}_r^*\vct{x}}^2\right)\frac{\omega^2(\mathcal{T})}{\delta^2},\frac{6}{\delta}-1\right),
\end{align}
then
\begin{align*}
\abs{\frac{1}{m}\sum_{r=1}^m(\vct{a}_r^*\vct{x})^2\left(\vct{a}_r^*(\vct{h}-(\vct{h}^*\vct{x})\vct{x})\right)^2-\twonorm{\vct{h}-(\vct{h}^*\vct{x})\vct{x}}^2}\le \frac{\delta}{4},
\end{align*}
holds with probability at least $1-6e^{-\frac{\delta^2}{1440}\left(\sum_{r=1}^m (\vct{a}_r^*\vct{x})^2\right)}$. Here, $\mathcal{T}$ denotes the projection of $\mathcal{C}\cap\mathbb{S}^{n-1}$ on the subspace orthogonal to the direction of $\vct{x}$. To bound the second term all that remains is to check that \eqref{sampE} holds. To this aim note that for a fixed vector $\vct{x}$
\begin{align*}
\max_{r}\text{ }\abs{\vct{a}_r^*\vct{x}}\le\sqrt{2\log n }\twonorm{\vct{x}},
\end{align*}
holds with probability at least $1-1/n$. Furthermore, by concentration of chi-squared random variables 
\begin{align*}
\frac{1}{m}\sum_{r=1}^m (\vct{a}_r^*\vct{x})^2\ge \frac{1}{2}.
\end{align*}
holds with probability at least $1-e^{-\gamma m}$ with $\gamma=0.01$. Therefore, using the fact that $\omega(\mathcal{T})\le\omega(\mathcal{C}\cap \mathbb{S}^{n-1})$, as long as 
\begin{align*}
m\ge \max\left(320\frac{\omega^2(\mathcal{C}\cap \mathbb{S}^{n-1})}{\delta^2}\log n,\frac{12}{\delta}-2\right),
\end{align*}
the second term of \eqref{main44} is bounded in absolute value by $\delta$ with probability at least $1-\frac{1}{n}-e^{\gamma_1 m}-6e^{-\gamma_2 \delta^2 m}$.

To bound the third term in \eqref{main44} note that we have
\begin{align*}
\frac{1}{m}\sum_{r=1}^m(\vct{a}_r^*\vct{x})^3\left(\vct{a}_r^*(\vct{h}-(\vct{h}^*\vct{x})\vct{x})\right)=\big\langle\frac{1}{m}\sum_{r=1}^m(\vct{a}_r^*\vct{x})^3\left(\vct{a}_r-(\vct{a}_r^*\vct{x})\vct{x}\right),\vct{h}-(\vct{h}^*\vct{x})\vct{x}\big\rangle.
\end{align*}
Now note that $\frac{1}{m}\sum_{r=1}^m(\vct{a}_r^*\vct{x})^3\left(\vct{a}_r-(\vct{a}_r^*\vct{x})\vct{x}\right)$ has the same distribution as $\left(\sqrt{\frac{1}{m^2}\sum_{r=1}^m(\vct{a}_r^*\vct{x})^6}\right)(\mtx{I}-\vct{x}\vct{x}^*)\vct{g}$ where $\vct{g}$ is a standard normal Gaussian random vector independent of the measurement vectors $\vct{a}_r$. So to bound the third term in \eqref{main44} it suffices to bound
\begin{align*}
\left(\sqrt{\frac{1}{m^2}\sum_{r=1}^m(\vct{a}_r^*\vct{x})^6}\right)\abs{\langle \vct{g},\vct{h}-(\vct{h}^*\vct{x})\vct{x}\rangle}.
\end{align*}
To bound this expression first note that by Chebyshev's inequality as long as $m\ge(2034)^2$,
\begin{align}
\label{mytempg}
\frac{1}{m}\sum_{r=1}^m(\vct{a}_r^*\vct{x})^6\le 20,
\end{align}
holds with probability at least $1-\frac{1}{m}$. Furthermore, by standard concentration of Gaussian processes
\begin{align*}
\abs{\langle \vct{g},\vct{h}-(\vct{h}^*\vct{x})\vct{x}\rangle}\le 2\omega(\mathcal{T})+\eta\le 2\omega(\mathcal{C}\cap\mathbb{S}^{n-1})+\eta,
\end{align*}
with probability at least $1-e^{-\frac{\eta^2}{2}}$. Now using $\eta=\frac{\delta}{16\sqrt{5}}\sqrt{m}$ in the above equation and combining it with \eqref{mytempg} we conclude that
\begin{align*}
\abs{\frac{1}{m}\sum_{r=1}^m(\vct{a}_r^*\vct{x})^3\left(\vct{a}_r^*(\vct{h}-(\vct{h}^*\vct{x})\vct{x})\right)}\le\frac{\sqrt{20}}{\sqrt{m}}(2\omega(\mathcal{C}\cap\mathbb{S}^{n-1})+\eta)=\frac{\sqrt{20}}{\sqrt{m}}\omega(\mathcal{C}\cap\mathbb{S}^{n-1})+\frac{\delta}{8},
\end{align*}
holds with probability at least $1-\frac{1}{m}-e^{-\frac{\delta^2}{2560}m}$. Therefore, as long as $m\ge1280\frac{\omega^2(\mathcal{C}\cap\mathbb{S}^{n-1})}{\delta^2}$, the third term with \eqref{main44} is bounded by $\frac{\delta}{2}$. Combining the bounds on the three terms concludes the proof of this lemma.

\subsection{Proof of Lemma \ref{RIP1}}
\label{ProofRIP1}
To prove this lemma we apply a powerful result of Mendelson in \cite{mendelson2012oracle}. Please also see \cite{eldar2014phase} for related calculations. To state this theorem we make use of two definitions. Define function classes
\begin{align*}
\mathcal{F}=&\{\abs{\langle\vct{u},\cdot\rangle}:\vct{u}\in \mathcal{T}_{\mathcal{F}}\}\quad\text{and}\quad\mathcal{H}=\{\abs{\langle\vct{u},\cdot\rangle}:\vct{u}\in \mathcal{T}_{\mathcal{H}}\}.
\end{align*}
Also for a set $\mathcal{T}$ define
\begin{align*}
d(\mathcal{T})=\underset{\vct{u}\in\mathcal{T}}{\sup}\twonorm{\vct{u}}^2.
\end{align*}
\begin{lemma}\cite{mendelson2012oracle}\label{mend} There exists absolute constants $c_1, c_2$ and $c_3$ for which the following holds. Let $\mathcal{T},\mathcal{H}\subset\R^n$ of cardinality at least 2 and set $\mathcal{F}$ and $\mathcal{H}$ to be the corresponding classes as defined above. Assume without loss of generality that $\frac{\omega(\mathcal{T}_{\mathcal{F}})}{d(\mathcal{T}_{\mathcal{F}})}\ge\frac{\omega(\mathcal{T}_{\mathcal{H}})}{d(\mathcal{T}_{\mathcal{H}})}$. For $t\ge c_1$, with probability at least
\begin{align*}
1-2\cdot\emph{exp}\left(-c_2t^2\min\Bigg\{m,\left(\frac{\omega(\mathcal{T}_\mathcal{F})}{d(\mathcal{T}_\mathcal{F})}\right)^2\Bigg\}\right),
\end{align*}
\begin{align*}
\underset{f\in\mathcal{F},\text{ }h\in\mathcal{H}}{\sup}\abs{\frac{1}{m}\sum_{r=1}^m f(\vct{a}_r)h(\vct{a}_r)-\mathbb{E} [fh]}\le c_3 t^2\left(td(\mathcal{T}_\mathcal{H})\frac{\omega(\mathcal{T}_\mathcal{F})}{\sqrt{m}}+\frac{\omega(\mathcal{T}_\mathcal{F})\omega(\mathcal{T}_\mathcal{H})}{m}\right).
\end{align*}
\end{lemma}
We shall use this theorem with 
\begin{align*}
\mathcal{T}_\mathcal{F}=\Big\{\frac{\vct{u}}{\twonorm{\vct{u}}}:\text{ }\vct{u}\in\mathcal{C}\Big\}\quad\text{and}\quad\mathcal{T}_\mathcal{H}=\Big\{\frac{\vct{v}}{\twonorm{\vct{v}}}:\text{ }\vct{v}\in\mathcal{C}'\Big\}.
\end{align*}
Note that $d(\mathcal{T}_{\mathcal{F}})=d(\mathcal{T}_{\mathcal{H}})=1$. The proof is complete by plugging the stated choice of $m$.

\subsection{Proof of Lemma \ref{expRIP1}}
\label{pfgordontype6}
Define $\vct{x}=\frac{1}{2}(\vct{u}+\vct{v})$ and $\vct{y}=\frac{1}{2}(\vct{u}-\vct{v})$. We have
\begin{align*}
\mathbb{E}\big[\abs{\vct{u}^*\vct{a}\vct{a}^*\vct{v}}\big]=\mathbb{E}\big[\abs{\abs{\vct{a}^*\vct{x}}^2-\abs{\vct{a}^*\vct{y}}^2}\big].
\end{align*}
Let $\vct{x}\vct{x}^*-\vct{y}\vct{y}^*=\lambda_1\vct{w}_1\vct{w}_1^*-\lambda_2\vct{w}_2\vct{w}_2^*$ be the eigen-value decomposition of $\vct{x}\vct{x}^*-\vct{y}\vct{y}^*$.
Therefore,
\begin{align}
\E[\abs{\abs{\vct{a}^*\vct{x}}^2-\abs{\vct{a}^*\vct{y}}^2}]=&\E[\abs{\vct{a}^*\left(\vct{x}\vct{x}^*-\vct{y}\vct{y}^*\right)\vct{a}}]\nonumber\\
=&\E[\abs{\lambda_1\abs{\vct{a}^*\vct{w}_1}^2-\lambda_2\abs{\vct{a}^*\vct{w}_2}^2}]\nonumber\\
=&\E[\abs{\lambda_1X_1^2-\lambda_2X_2^2}],
\end{align}
where $X_1,X_2\in\R$ are independent $\mathcal{N}(0,1)$ random variables. Before we proceed further let us now calculate the eigenvalues $\lambda_1,\lambda_2$ and verify that $\lambda_1,\lambda_2\ge 0$. Without loss of generality we can assume
\begin{align*}
\vct{x}=\twonorm{\vct{x}}\vct{e}_1\quad\text{and}\quad\vct{y}=\twonorm{\vct{y}}\left(\rho\vct{e}_1+\sqrt{1-\rho^2}\vct{e}_2\right),
\end{align*}
where $\rho=\frac{\vct{x}^*\vct{y}}{\twonorm{\vct{x}}\twonorm{\vct{y}}}$. Since $\vct{x}\vct{x}^*-\vct{y}\vct{y}^*$ has rank two its eigenvalues are the same as the matrix
\begin{align*}
\begin{bmatrix}
\twonorm{\vct{x}}^2-\rho^2\twonorm{\vct{y}}^2& -\rho\sqrt{1-\rho^2}\twonorm{\vct{y}}^2\\
-\rho\sqrt{1-\rho^2}\twonorm{\vct{y}}^2 & (\rho^2-1)\twonorm{\vct{y}}^2.
\end{bmatrix}
\end{align*}
Algebraic manipulations show that the eigenvalues of the above matrix are $\lambda_1,-\lambda_2$ where
\begin{align*}
\lambda_1=&\frac{1}{2}\left(\twonorm{\vct{x}}^2-\twonorm{\vct{y}}^2+\sqrt{\left(\twonorm{\vct{x}}^2-\twonorm{\vct{y}}^2\right)^2+4(1-\rho^2)\twonorm{\vct{x}}^2\twonorm{\vct{y}}^2}\right),\\
=&\frac{1}{2}\left(\twonorm{\vct{x}}^2-\twonorm{\vct{y}}^2+\twonorm{\vct{x}-\vct{y}}\twonorm{\vct{x}+\vct{y}}\right),\\
\lambda_2=&\frac{1}{2}\left(\twonorm{\vct{y}}^2-\twonorm{\vct{x}}^2+\sqrt{\left(\twonorm{\vct{x}}^2-\twonorm{\vct{y}}^2\right)^2+4(1-\rho^2)\twonorm{\vct{x}}^2\twonorm{\vct{y}}^2}\right),\\
=&\frac{1}{2}\left(\twonorm{\vct{y}}^2-\twonorm{\vct{x}}^2+\twonorm{\vct{x}-\vct{y}}\twonorm{\vct{x}+\vct{y}}\right).
\end{align*}
It is easy to verify that $\lambda_1,\lambda_2\ge 0$. Set
\begin{align*}
\eta=\frac{\lambda_1-\lambda_2}{\lambda_1+\lambda_2}=\frac{\twonorm{\vct{x}}^2-\twonorm{\vct{y}}^2}{\twonorm{\vct{x}-\vct{y}}\twonorm{\vct{x}+\vct{y}}}=\frac{\vct{u}^T\vct{v}}{\twonorm{\vct{u}}\twonorm{\vct{v}}}\quad\text{and}\quad \cos \theta =\eta,
\end{align*}
in which $\theta\in[0,\pi]$. Similar to the argument on page 13 of \cite{candes2012phaselift} by using polar coordinates, we have
\begin{align*}
\E[\abs{\abs{\vct{a}^*\vct{x}}^2-\abs{\vct{a}^*\vct{y}}^2}]=\E[\abs{\lambda_1X_1^2-\lambda_2X_2^2}]=&\frac{1}{2\pi}\left(\int_0^\infty r^3e^{-r^2/2}dr\right)\left(\int_0^{2\pi}\abs{\lambda_1\cos^2\phi-\lambda_2\sin^2\phi}d\phi\right)\\
=&\frac{1}{\pi}\int_0^{2\pi}\abs{\lambda_1\cos^2\phi-\lambda_2\sin^2\phi}d\phi\\
=&\frac{2}{\pi}\int_0^\pi\abs{\lambda_1\cos^2\phi-\lambda_2\sin^2\phi}d\phi.
\end{align*}
Now using the identities $\cos^2\phi=(1+\cos 2\phi)/2$ and $\sin^2\phi=(1-\cos 2\phi)/2$, we have
\begin{align*}
\E[\abs{\abs{\vct{a}^*\vct{x}}^2-\abs{\vct{a}^*\vct{y}}^2}]=&\frac{\lambda_1+\lambda_2}{\pi}\int_0^\pi\abs{\cos 2\phi+\eta}d\phi\\
=&\frac{\lambda_1+\lambda_2}{2\pi}\int_0^{2\pi}\abs{\cos \phi+\eta}d\phi\\
=&\frac{\lambda_1+\lambda_2}{2\pi}\left(\int_0^{\pi}\abs{\cos \phi+\eta}d\phi+\int_0^{\pi}\abs{-\cos \phi+\eta}d\phi\right)\\
=&\frac{2}{\pi}(\lambda_1+\lambda_2)\left(\sqrt{1-\eta^2}+\eta\sin^{-1}(\eta)\right)\\
=&\frac{2}{\pi}\twonorm{\vct{u}}\twonorm{\vct{v}}\left(\sin\theta+\cos\theta(\pi/2-\theta)\right).
\end{align*}

\end{document}